\documentclass[letterpaper]{article} 
\usepackage[preprint]{aaai2027}  
\usepackage[hyphens]{url}  
\usepackage{graphicx} 
\usepackage{natbib}  
\usepackage{caption} 
\usepackage{algorithm}
\usepackage{algorithmic}
\usepackage{scalerel}

\usepackage{newfloat}
\usepackage{listings}
\DeclareCaptionStyle{ruled}{labelfont=normalfont,labelsep=colon,strut=off} 
\floatstyle{ruled}
\newfloat{listing}{tb}{lst}{}
\floatname{listing}{Listing}

\usepackage{booktabs}

\nocopyright

\newcommand{\paperid}{9709}
\newcommand{\papertitle}{Recurrent Graph Neural Networks with Set-Based Aggregation}
\title{\papertitle}
\author{Blai Bonet}

\affiliations{
  Universidad Sim\'on Bol\'{\i}var \\
  Caracas, Venezuela \\
  \url{bonetblai@gmail.com}

}

\usepackage{amsmath,amssymb,amsthm,bm}
\usepackage{euscript}
\usepackage{mathrsfs}
\usepackage{mathtools}
\usepackage{booktabs}
\usepackage{enumitem}
\usepackage{xspace}
\usepackage{multirow}
\usepackage{relsize}
\usepackage{scalefnt}
\usepackage{makecell}

\usepackage[appendix=append, bibliography=common]{apxproof}

\makeatletter
\newcommand{\myappendixheader}[1]{%
  \par\vspace{2em}
  \noindent\textbf{\Large #1}
  \par\vspace{1em}
  \@afterheading%
}
\makeatother

\usepackage{tocloft}
\newcommand{\Omit}[1]{}
\newcommand{\tup}[1]{\ensuremath{\langle #1 \rangle}\xspace}
\newcommand{\set}[1]{\ensuremath{\{ #1 \}}\xspace}

\newcommand{\bracket}[1]{\ensuremath{[\![ #1 ]\!]}\xspace}

\newcommand{\powerset}{\raisebox{.15\baselineskip}{\ensuremath{\wp}}}

\newcommand{\R}{\ensuremath{\mathbb{R}}}

\newtheoremrep{theorem}{Theorem} 
\newtheoremrep{lemma}[theorem]{Lemma}
\newtheoremrep{corollary}[theorem]{Corollary}
\newtheoremrep{proposition}[theorem]{Proposition}
\newtheoremrep{definition}[theorem]{Definition}
\newtheoremrep{remark}[theorem]{Remark}
\newtheoremrep{example}[theorem]{Example}
\newtheoremrep{openquestion}[theorem]{Open Question}
\newtheoremrep{conjecture}[theorem]{Conjecture}

\newtheoremrep{innerlemma}{Lemma}

\newtheorem{apxdefinition}{Definition}
\newtheorem{apxtheorem}[apxdefinition]{Theorem}
\newtheorem{apxproposition}[apxdefinition]{Proposition}

\colorlet{darkred}{red!80!black}
\colorlet{olive}{green!60!blue}
\newcommand{\alert}[1]{\textcolor{darkred}{#1}}

\newcommand{\param}{\ensuremath{\theta}\xspace}

\newcommand{\ReLU}{\mathrm{ReLU}}
\newcommand{\CLIP}[2]{\mathrm{clip}_{[#1,#2]}}
\newcommand{\clip}{\CLIP{0}{1}}
\newcommand{\Net}{\mathrm{Net}}

\newcommand{\types}[1]{\scalebox{0.9}{$\mathsf{#1}$}}
\newcommand{\atoms}{\types{ATP}}
\newcommand{\atp}{\mathsf{atp}}
\newcommand{\TP}{\types{TP}}
\newcommand{\tp}[1]{\TP^{(#1)}}
\newcommand{\ftype}[1]{\tau^{(#1)}}

\newcommand{\Embeddings}{\ensuremath{\R^W}\xspace}
\newcommand{\EF}[1]{\mathbb{E}_{#1}}

\newcommand{\tpw}{\TP^{(\omega)}}
\newcommand{\itw}{\bar{\bm{t}}}
\newcommand{\itype}{\tau}

\newcommand{\initial}[1]{\bm{e}_{#1}}
\newcommand{\Eop}[1]{F_{#1}}

\newcommand{\Table}{\mathcal{T}}
\newcommand{\GEmb}[1]{\bm{E}^{#1}}
\newcommand{\GEmbu}[2]{\bm{E}^{#1}_{#2}}

\newcommand{\FO}{\mathrm{FO}}
\newcommand{\ifp}[2]{[ \mathrm{\bf ifp}_{#1} #2(#1) ]}
\newcommand{\Au}{\ensuremath{\mathfrak{A}}}

\newcommand{\maxnbr}[1]{\mathrm{maxnbr}_{#1}}
\newcommand{\curr}[1]{\mathrm{curr}_{#1}}

\renewcommand{\S}{\ensuremath{\mathcal{S}}\xspace}

\newcommand{\T}{\ensuremath{\mathcal{T}}\xspace}

\newcommand{\V}{\ensuremath{\mathcal{V}}\xspace}
\Omit{

  \newcommand{\T}{\ensuremath{\mathcal{T}}\xspace}

}

\newcommand{\bTheta}{\boldsymbol{\Theta}}

\newcommand{\Lmu}{\ensuremath{\text{L}\mu}\xspace}

\newcommand{\mybox}{\scalebox{0.80}{$\Box$}\xspace}
\newcommand{\mydiamond}{\raisebox{1pt}{\scalebox{0.80}{$\Diamond$}}\xspace}
\newcommand{\sd}{\ensuremath{\Sigma^{\diamond}_{1}}\xspace}   
\newcommand{\pb}{\ensuremath{\Pi^{\scalebox{0.50}{$\Box$}}_{1}}\xspace}          
\newcommand{\bsd}{\ensuremath{\mathrm{\bf B}\sd}\xspace}      
\newcommand{\ratchet}{\textsc{Ratchet}\xspace}

\begin{document}

\maketitle

\begin{abstract}
  Recurrent GNNs iterate message passing to convergence,
  and their logical characterizations to date rely on multiset aggregation, graded
  (counting) logics, and halting or acceptance conditions that cannot be verified
  from the network's parameters.
  We study recurrent GNNs with set-based aggregation and identify sufficient conditions
  checkable from the weights for networks to compile into formulas and formulas into networks.
  The main result is an effective, two-directional equivalence between a class of
  networks and the Boolean closure of reachability and safety properties, the
  fragment \bsd of the modal $\mu$-calculus.
  The fragment is not an artifact: it is the exact expressive level of
  stabilization over finite vocabulary, which supports fixed points of a single polarity
  and Boolean combinations thereof, but not the composition of fixed points of opposite
  polarities. 
  The correspondence needs no counting logic, no external halting signal, and no
  non-effective acceptance condition, yielding a verifiable path from weights to
  symbolic explanations for networks meeting the conditions.
\end{abstract}

\begin{toappendix}
  \section{Suppl.\ Material for Paper \#\paperid: \papertitle}
  \bigskip
\end{toappendix}

\section{Introduction}

Graph neural networks (GNNs) \cite{gori:gnn} admit a precise reading as logic: a trained network can
be compiled into formulas that describe what it computes, and formulas can be compiled
into networks that realize them.
For depth-bounded nets this correspondence is well understood, with first-order logic
with counting characterizing exactly what such networks compute
\cite{xu:gnn,morris:gnn,barcelo:gnn,grohe2024descriptive}.
For recurrent GNNs, which iterate message passing to convergence rather than for a fixed
number of layers, the picture is incomplete.
Existing characterizations \cite{pflueger2024:recurrent,Ahvonenetal2024,bollen2025:halting,rosenbluth2026repetition}
are built for multiset aggregation, and their target logics, graded modal logic and the
graded $\mu$-calculus, inherit its ability to count.
This power is not free: the correspondences rest on explicit halting classifiers
or on acceptance conditions that cannot be verified from the network's weights,
and for the graded $\mu$-calculus, the setting with genuine alternation,
the converse direction is known only conditionally
\cite{bollen2025:halting,bollen2026:halting-vs-converging}.

We study recurrent GNNs with set-based aggregation, such as max and min, 
and show that this restriction yields what the multiset
setting so far lacks: a complete, effective, two-directional correspondence with a fixed-point
logic, under conditions checkable from the weights.
The price is counting and alternation, and we prove the price is unavoidable. 
Concretely, we identify regularity, a semantic condition under which a converging net's
fixed-point embedding at a vertex depends only on its atomic type and its neighbors' fixed points,
not on the transient trajectory; the vocabulary, the set of embeddings reachable at convergence,
which can be finite even when the transients are infinite in number; and admissibility,
under which the net's run can be traded for a finite dynamical system (the game) iterating
directly over the vocabulary.
Converging regular nets with finite vocabulary whose game is admissible, ordered, and monotone
are called \ratchet nets.
Our contributions are the following:
\begin{enumerate}[leftmargin=*, label=$\bullet$]
  \item A checkable framework for compiling recurrent nets.
    We introduce regularity, vocabulary, and admissibility, and give sufficient conditions
    verifiable from the weights: contraction (Theorem~\ref{thm:contractive}) and monotonicity
    (Theorem~\ref{thm:monotone}).
    Unlike the lim-inf and ever-reached acceptance conditions of prior work, membership in the
    resulting class does not depend on non-effective properties of infinite runs.
  \item A general-purpose compiler into fixed-point logic.
    Every regular net with finite vocabulary and an admissible, ordered game compiles into a
    simultaneous inflationary fixed-point (IFP) formula of size linear in the net's quotient table,
    with no monotonicity required (Theorem~\ref{thm:rec-compiler}).
  \item A two-way correspondence.
    \ratchet nets compute exactly the vertex queries definable in \bsd, the Boolean closure of
    reachability (least fixed points over $\mydiamond$) and safety 
    (greatest fixed points over $\mybox$) in the modal $\mu$-calculus, with effective translations
    in both directions (Theorems~\ref{thm:net-to-logic} and \ref{thm:logic-to-net}).
    The net-to-logic direction answers, for this counting-free fragment, the converse question
    that \citet{bollen2025:halting} leave open in the graded setting; the logic-to-net direction
    is the counterpart of their Theorem 5.1, obtained without their halting classifier or counting
    apparatus.
  \item A barrier showing \bsd is necessary. 
    No converging regular net with finite vocabulary and admissible game,  ordering and monotonicity
    are not even assumed, computes $\nu Y.\mydiamond Y$, $\mu X.\mybox X$, or the alternation-free $\text{EF EG }p$
    (Proposition~\ref{prop:barrier}).
    Stabilization over finite vocabulary supports fixed points of a single polarity and their Boolean
    combinations, but cannot certify that an inner fixed point of opposite polarity has stabilized.
    Indeed, \bsd omits queries from the first level of the alternation hierarchy and is
    strictly contained in the alternation-free fragment (Corollary~\ref{cor:separation}).
  \item Two implemented compilers, net-to-logic and logic-to-net, that map \ratchet nets
    to \bsd, and back, together with an extensive experimental validation.
\end{enumerate}

The restriction to set-based aggregation is not arbitrary. Aggregators such as max and min are precisely
the empirically preferred choice in neural algorithmic reasoning: max-aggregation message passing was
found best suited for learning to execute classical graph algorithms and for size generalization
\cite{velickovic2020neural,morris:compute};
this preference has been explained  \cite{xu2020what,xu2021how,dudzik2022gnns},
and it is the aggregation adopted by the generalist algorithmic learner of
\citet{ibarz2022generalist}.
Our results give this architectural choice logical meaning.

The correspondence locates set-based recurrent GNNs within the landscape of stopping semantics.
A fixed point built on top of one of the opposite polarity requires a certificate that the inner
computation has stabilized: supplied externally by a halting classifier \cite{bollen2025:halting},
or manufactured internally by counters that grow with the graph \cite{bollen2026:halting-vs-converging}.
Proposition~\ref{prop:barrier} shows that with neither resource available, no certificate exists,
and \bsd is exactly what remains; details below. 

Roadmap. After fixing the model and vertex types, 
we develop the recurrent case, \ratchet nets, and the IFP compiler. 
The central section establishes the equivalence with \bsd and the barrier.
We close with related work and conclusions.

\section{GNNs with Set-Based Aggregation}
\label{sec:gnns}

Graphs are denoted with the letters $x, y, \ldots$, and
vertices with $u, v, \ldots$.
The vertex and edge set for graph $x$ are denoted by $V(x)$ and $E(x)$
respectively, $E(x,u,v)$ is the edge relation (or $E(u,v)$ when
$x$ is clear from context), and $N(u) = N(x,u) = \{v \in V(x) \mid E(x,u,v)\}$
is vertex $u$'s neighborhood.
A finite and non-empty set $\atoms$ of \emph{atomic types} that
partitions the vertices on every graph is \textbf{fixed throughout:}
$\atp(u)$ is the atomic type of vertex $u$.
A graph $x$ is a relational structure with universe $V(x)$,
edge relation $E$, and unary relations $\set{ \alpha(\cdot) \mid \alpha\,{\in}\,\atoms}$.
A \emph{pointed} graph is a pair $(x,u)$ for 
graph $x$ and distinguished $u\,{\in}\,V(x)$.

A feature or \emph{embedding} is a vector in $\Embeddings$ whose
components are called \emph{registers}.
For atomic type $\alpha$ and layer index $\ell$, the GNN architecture fixes:
\begin{enumerate}[leftmargin=*, label=$\bullet$]
  \item an \emph{initial embedding} $\initial\alpha \in \Embeddings$,
  \item weight tensors $\bm{A}_{\alpha,\ell}, \bm{B}_{\alpha,\ell} \in \R^{W{\times}W}$
    and $\bm{C}_{\alpha,\ell} \in \Embeddings$,
  \item combination function $\oplus_\ell: \powerset(\Embeddings) \to \Embeddings$
    that maps \textbf{finite} subsets to embeddings, and
  \item component-wise \emph{activation functions} $\rho_\ell : \R \to \R$.
\end{enumerate}
For simplicity, we assume architectures with equal combination and activation
functions across layers, but this is not an essential distinction.
For activation functions, we consider $\ReLU(x)\,{=}\,\max(0,x)$
and the \emph{truncated ReLU} (or clipped linear unit) $\clip(x)\,{=}\,\min(1, \max(0,x))$.
These tensors and functions make up the parameter $\param$ of the net.

The central object is the one-step update at layer $\ell$:
\begin{equation}
  \label{eq:Net}
  \Net_\ell(\alpha, z, S) \;=\;
    \rho\!\left( \bm{A}_{\alpha,\ell}\, z \,+\, \bm{B}_{\alpha,\ell}\oplus\!S \,+\, \bm{C}_{\alpha,\ell} \right)
\end{equation}
where $\alpha\,{\in}\,\atoms$ is an atomic type, $z\,{\in}\,\Embeddings$ is the
current embedding of a vertex, $S\,{\subseteq}\,\Embeddings$ is the (finite) set of
neighbor embeddings, and $\oplus\!S$ is their aggregation,
with the convention $\oplus\emptyset = \mathbf{0}$ (the zero vector),
so that~\eqref{eq:Net} is well defined for isolated vertices.
This is a concrete architecture choice but others are also compatible with our results.

In the depth-bounded or non-recurrent case, the net is defined
for a fixed number $L$ of layers. The net iterates the updates synchronously
over all vertices.
Starting from $u^{(0)} = \initial{\atp(u)}$,
the embedding of vertex $u$ in graph $x$ at layer $\ell{+}1$ is
\[ u^{(\ell+1)} \;=\; \Net_\ell\bigl( \atp(u),\, u^{(\ell)},\, \{v^{(\ell)} \mid E(x,u,v)\} \bigr)\,. \]
We write $u^{(\ell)}_{x,\theta}$ when the dependence on the graph and
parameter must be made explicit.
In the recurrent case, we assume equal weight tensors across layers
and denote the one-step function simply as $\Net(\alpha,z,S)$.
The architecture is \emph{converging} if
$u^{(\ell)}_{x,\theta}$ converges in $\Embeddings$ as $\ell\to\infty$,
for every graph $x$ and vertex $u$, in which case the limit is
denoted as $u^*_{x,\theta}$.
Convergence of the iteration is non-trivial and depends on the
interaction between the weight matrices, the activation function, and
the graph structure.

For simplicity, the recurrent net iterates a single-layer block. 
Our results extend to $L$-layer recurrent blocks, where an entire iteration
of the block 
replaces a single-step update.

Notice that graphs are directed: $E$ is an arbitrary binary relation, $N(u)$
is the set of $E$-successors of $u$, and no symmetry is assumed; undirected
graphs are the special case of symmetric $E$, so all positive results below
apply to them.

\section{Finite Types and Their Embeddings}
\nosectionappendix
\begin{toappendix}
  \myappendixheader{Appendix for section: Finite Types and Their Embeddings}
\end{toappendix}

The vertex types correspond to the different colors obtained by 1-WL under
\emph{set aggregation} \cite{weisfeiler1968reduction,Immerman92:C2-WL,Grohe17:descriptive}.
Recall that $\atoms$ is the fixed set of atomic types.

\begin{definition}[Types]
  \label{def:types}
  \label{def:vertex-types}
  The sequence $\tup{ \tp\ell}_{\ell\geq0}$ of types is:
  $\tp0\,{=}\,\atoms$, and $\tp{\ell+1}\,{=}\,\set{ (t, S) \mid t\,{\in}\,\tp\ell, S\,{\subseteq}\,\tp\ell }$.
  \Omit{
  \begin{alignat*}{1}
    \tp0        &=  \atoms \,, \\
    \tp{\ell+1} &=  \set{ (t, S) \mid t \in \tp\ell, S \subseteq \tp\ell } \,.
  \end{alignat*}
}
  The \emph{atomic type} of type $t$ is $\atp(\alpha)=\alpha$ for $\alpha\in\atoms$,
  and $\atp(t)=\atp(t')$ for $t=(t',S')\in\tp\ell$ with $\ell\geq 1$.

  The (finite) type $\ftype{\ell}(x,u)$ for pointed graph $(x,u)$ and $\ell\geq0$ is
  given by $\ftype0(x,u)\,{=}\,\atp(u)$, and $\ftype{\ell+1}(x,u)\,{=}\,(t,S)$ iff
  $\ftype\ell(x,u)\,{=}\,t$ and $S\,{=}\,\set{ \ftype\ell(x,v) \mid E(u,v) }$.
\end{definition}

\Omit{
  \begin{definition}[Vertex types]
    \label{def:vertex-types}
    Let $(x,u)$ be a pointed graph.
    The vertex types $\ftype\ell(x,u)$, for $\ell\geq0$, are the following:
    $\ftype0(x,u)=\atp(u)$, and 
    $\ftype{\ell+1}(x,u)=(t,S)$ iff $\ftype\ell(x,u) = t$ and $S = \set{ \ftype\ell(x,v) \mid E(u,v) }$.
  \end{definition}
}

The types at layer $\ell$ correspond to the equivalence classes of an $\ell$-bisimulation
relation under set-based aggregation: two vertices of the same type obtain the same
embeddings regardless of the architecture.
Hence, the embeddings that a depth-bounded net with set-based aggregation
can generate are finite and determined by the types:

\begin{definition}
  \label{def:type-embedding}
  The embedding function $\EF{t}: \bTheta \rightarrow \Embeddings$, one per type $t$,
  maps parameters into embeddings:
  \begin{enumerate}[leftmargin=*, label=\arabic*.]
    \item for $\alpha\,{\in}\,\atoms$, $\EF{\alpha}(\param) = \initial\alpha$
      is a constant function of $\param$, and
    \item for $(t, S)$ in $\tp{\ell+1}$ where $t$ is of atomic type $\alpha$,
      $\EF{(t, S)}(\param)\,{=}\,\Net_{\ell}(\alpha, \EF{t}(\param), \set{ \EF{t'}(\param) \mid t' \in S } \bigr)$.
  \end{enumerate}
  When $\param$ is clear from context, we write $\EF{t}$ instead of $\EF{t}(\param)$.
\end{definition}

\begin{theoremrep}
  \label{thm:embedding-determination}
  Let $\param$ be a parameter,
  let $(x,u)$ be a pointed graph, and let $\ell\geq0$ be a layer index within the architecture.
  Then, $u^{(\ell)}_{x,\theta}\,{=}\,\EF{t}(\theta)$ for $t=\ftype\ell(x,u)$.
\end{theoremrep}
\begin{proof}
  By direct induction on the layer index $\ell\geq0$.
\end{proof}


\subsection{Compiling Depth-Bounded Nets}

A direct byproduct of the relation between types and embeddings is
a ``compiler'' that maps nets into logical formulas that characterize
the embeddings. 
%
The idea is to have logical formulas that decide the types, 
and use the parameter-to-embedding function to logically characterize the net.
As the embeddings are real vectors, we settle for a logical characterization
that test whether a register of the final embedding is greater than zero,
but other criteria are possible.

\begin{definition}
  \label{def:type-formula}
  For each type $t$, there is formula $\tau_t(z)$ that holds in pointed graph
  $(x,u)$ iff $\ftype\ell(x,u)\,{=}\,t$ where $\ell$ is the layer of $t$:
  \begin{enumerate}[leftmargin=*, label=$\bullet$]
    \item 
      $\tau_\alpha(z) = \alpha(z)$, and
    \item 
      $\tau_{(t, S)}(z) =
          \tau_t(z) \land
          \bigl[ \forall w \bigl( E(z,w) \,{\implies}\, \bigvee_{t'{\in}S} \tau_{t'}(w) \bigr) \bigr] \land
          \bigl[ \bigwedge_{t'{\in}S} \exists w \bigl( E(z,w) \land \tau_{t'}(w) \bigr) \bigr]$.
  \end{enumerate}
\end{definition}

Per definition, vertex $u$ satisfies $\tau_{(t,S)}$ iff its type is $t$, each neighbor
has a type in $S$, and each type in $S$ is witnessed by some neighbor.
A logical characterization of embeddings follows whose soundness is established by 
induction on $\ell$:

\begin{definition}
  Let $\param$ be a parameter for the architecture $(W,L,\atoms)$ of width $W$, depth $L$,
  and atomic types in $\atoms$. 
  The logical characterization of the embeddings at layer $\ell$ is the \textbf{vector}
  \[ \Psi^{(\ell)}(z)\,{=}\,\tup{\psi^{(\ell)}_1(z), \ldots, \psi^{(\ell)}_W(z)} \]
  where
  $\psi^{(\ell)}_i(z) \,{=}\,
      \bigwedge_{t \in \tp\ell} \bigl( \tau_t(z) \,{\implies}\, \bracket{ [ \mathbb{E}_t(\param) ]_i \,{>}\, 0 } \bigr)$.
  \Omit{
    Likewise, the \textbf{vector}
    \[ \Phi^{(\ell)} = \tup{\phi^{(\ell)}_1, \ldots, \phi^{(\ell)}_W} \]
    where
    $\phi^{(L)}_i \,{=}\,
        \bigwedge_{S \subseteq \tp\ell} ( \gamma_S \,{\Rightarrow}\, \bracket{ [ \mathbb{G}_S(\param) ]_i \,{>}\, 0 } )$
    characterizes the graph embeddings at layer $\ell$.
  }
\end{definition}

\begin{theoremrep}[Compiler]
  \label{thm:bdd-compiler}
  Let $\theta$ be a parameter of the architecture $(W,L,\atoms)$.
  Then, $[u^{(\ell)}_{x,\param}]_i\,{>}\,0$ iff $x \vDash \psi^{(\ell)}_i(u)$.
\end{theoremrep}
\begin{proof}
  By direct induction on the layer index $\ell\geq0$.
\end{proof}

The size of $\tp\ell$ is a tower of exponentials, 
related to the number of 
Hintikka-style normal forms of modal formulas of modal depth $\leq\ell$,
supporting the blow-up empirically reported \cite{tenacucala2024bridging,tena2022explainable}
for similar compilations.

\section{Recurrent Nets}
\nosectionappendix
\begin{toappendix}
  \myappendixheader{Appendix for section: Recurrent Nets}
\end{toappendix}

Getting a characterization in the recurrent case is harder:
the net converges in a number of steps that is not bounded, and the
set of types the net may see as well as the number of fixed-point
embeddings it may reach are also unbounded.

The embedding for pointed graph $(x,u)$ is the
limit $u^*_x$,  which by Theorem~\ref{thm:embedding-determination},
equals $\lim_{\ell\to\infty} \EF{\ftype\ell(x,u)}$.
Its \textbf{infinite type} $\itype(x,u)$ 
is the sequence $\tup{\ftype\ell(x,u)}_{\ell\geq0}$ of finite types,
while the function 
$\Phi(\cdot)$ maps infinite types 
into (limit) embeddings:
$\Phi(\itype(x,u)) \doteq \lim_{\ell\to\infty} \EF{\ftype\ell(x,u)}$.
An infinite type is an equivalence class for a bisimulation
that uses set-based aggregation of neighbor types.
The set $\tpw=\set{ \itype(x,u) \mid \text{$(x,u)$ is pointed graph} }$
is the set of infinite types. 
For converging nets, $\Phi$ is a well-defined and total function
on $\tpw$. The following is direct:

\begin{lemmarep}[Bisimulation invariance]
  \label{lemma}
  Let $(x,u)$ and $(x',u')$ be two pointed (finite) graphs.
  Then, $(x,u)$ and $(x',u')$ are bisimilar, written as $(x,u)\sim(x',u')$,
  iff $\itype(x,u)=\itype(x',u')$.
  Hence, if $(x,u)\sim(x',u')$, then $\Phi(\itype(x,u))=\Phi(\itype(x',u'))$.
\end{lemmarep}
\begin{proof}
  A bisimulation between $x$ and $x'$ is a relation $Z\subseteq V(x)\times V(x')$
  such that for every $(u, u') \in Z$:
  \begin{enumerate}[leftmargin=*, label=$\bullet$]
    \item (atoms) $\atp(u) = \atp(u')$;
    \item (forth) for every $v$ with $E(x, u, v)$ there is $v'$ with $E(x', u', v')$ and $(v, v') \in Z$; and
    \item (back) symmetrically.
  \end{enumerate}
  The pointed structues $(x,u)$ and $(x',u')$ are bisimilar, $(x,u)\sim(x',u')$, iff
  there is a bisimulation $Z\subseteq V(x)\times V(x')$ such that $(u,u')\in Z$.
  We use the fact that types refine: $\ftype{\ell+1}(x,u)$ determines $\ftype{\ell}(x,u)$,
  as the latter is the first component of the former.
  Hence agreement of types at level $\ell$ implies agreement at of types at every level below $\ell$.

  \medskip\noindent
  $(\Rightarrow)$. Let $Z$ be a bisimulation. 
  We show by induction on $\ell$ that $\ftype{\ell}(x,u) = \ftype{\ell}(x',u')$
  for every $(u,u') \in Z$.
  The base case $\ell=0$ is direct by the atom condition: $\ftype{0}(x,u)=\atp(u)=\atp(u')=\ftype{0}(x',u')$.
  For the inductive step, write $\ftype{\ell+1}(x,u) = (\ftype{\ell}(x,u), S)$
  and $\ftype{\ell+1}(x',u') = (\ftype{\ell}(x',u'), S')$.
  The first components agree by inductive hypothesis.
  For every $v$ with $E(x, u, v)$, the forth condition yields $v'$ with
  $E(x', u', v')$ and $(v, v') \in Z$, so $\ftype{\ell}(x,v) = \ftype{\ell}(x',v')$
  by the inductive hypothesis; hence $S \subseteq S'$.
  The back condition gives $S' \subseteq S$.
  Thus $\ftype{\ell+1}(x,u) = \ftype{\ell+1}(x',u')$.
  Therefore, if $(u,u') \in Z$, then $\itype(x, u) = \itype(x',u')$.

  \medskip\noindent
  $(\Leftarrow)$. Define $Z = \set{(u,u') \in V(x) \times V(x') \mid \itype(x,u) = \itype(x',u')}$.
  We verify $Z$ is a bisimulation.
  Let $(u, u') \in Z$; i.e., $\itype(x,u)=\itype(x',u')$.
  The (atoms) condition holds since $\atp(u)=\ftype{0}(x,u) = \ftype{0}(x',u')=\atp(u')$.
  For (forth), let $v$ such that $E(x,u,v)$.
  For each $\ell\geq0$, equality $\ftype{\ell+1}(x,u) = \ftype{\ell+1}(x',u')$
  gives equality of the neighbor-type sets at level $\ell$,
  so there is a neighbor $v'_\ell$ of $u'$ with $\ftype{\ell}(x',v'_\ell) = \ftype{\ell}(x,v)$.
  Since $v'$ has \emph{finitely many neighbors,} some fixed neighbor $v'$ occurs as $v'_\ell$
  for infinitely many $\ell$.
  By refinement, $\ftype{\ell}(x',v') = \ftype{\ell}(x,v)$ then holds for every $\ell\geq 0$;
  i.e., $\itype(x',v') = \itype(x,v)$ and thus $(v,v')\in Z$.
  The (back) condition is symmetric.

  \medskip
  Finally, $\Phi$ is by definition a function of the infinite type,
  $\Phi(\itype(x,u)) = \lim_{\ell\to\infty} \EF{\ftype{\ell}(x,u)}$.
  Thus, bisimilar (finite) pointed graphs, having equal infinite types, receive equal limit embeddings.
\end{proof}

At convergence, the net reaches an equilibrium where the vertices $u$
of the graph satisfy the system:
\[ u^* = \Net(\atp(u), u^*, \set{ v^* \mid E(x,u,v)}) \,. \] 
As the neighbors ``shield'' $u$ from other vertices in the graph,
it is natural to expect for $u^*$ to be \emph{determined} by its atomic type $\atp(u)$
and $S=\set{ v^* \mid E(x,u,v)}$. 
We interpret determination in this context as $u^* = \lim_{\ell\to\infty} z_\ell$
where the sequence $\tup{z_\ell}_{\ell\geq 0}$ satisfies $z_0 = \initial\alpha$
and $z_{\ell+1} = \Net(\alpha, z_\ell, S)$ for $\alpha=\atp(u)$.
Another notion asks for $u^*$ to be the \emph{unique fixed point} of
$z \mapsto \Net(\alpha, z, S)$, but this is too strong
and doesn't cover interesting cases. 
However, even this weaker notion is not always satisfied,
and thus motivates: 

\begin{definition}[Regular nets]
  \label{def:regular}
  Let $\Net(\cdot, \cdot, \cdot)$ be a \textbf{converging} recurrent architecture
  for the atomic types in $\atoms$.
  The net is \textbf{regular} iff for every pointed structure $(x,u)$ of type $\alpha$:
  \[ \Phi(\itype(x,u)) = \Eop\alpha(S) \doteq \lim_{\ell\to\infty} z_\ell \]
  where $S = \set{ \Phi(\itype(x,v)) \mid E(x,u,v) }$, and $\tup{z_\ell}_{\ell\geq0}$
  satisfies $z_0=\initial\alpha$ and $z_{\ell+1} = \Net(\alpha, z_\ell, S)$.
\end{definition}


\begin{theoremrep}
  \label{thm:bisimulation}
  Let $\Net(\cdot, \cdot, \cdot)$ be a regular net for the atomic types in $\atoms$,
  and let $(x,u)$ and $(x',u')$ be pointed graphs of the same atomic type $\alpha\in\atoms$.
  If $\set{\Phi(\itype(x,v)) \mid E(x,u,v)} = \set{\Phi(\itype(x',v')) \mid E(x',u',v')}$,
  then $\Phi(\itype(x,u)) = \Phi(\itype(x',u'))$.
\end{theoremrep}
\begin{proof}
  Direct by definition of regular nets.
\end{proof}

The \textbf{vocabulary} of the net is constructed by closing the initial
embeddings with the operators in $\set{ \Eop\alpha \mid \alpha \in \atoms }$.
For the sequence $\tup{\V_k}_{k\geq 0}$ where $\V_0=\set{\Eop\alpha(\emptyset) \mid \alpha\in\atoms}$
is the set of fixed-point embeddings for isolated vertices, and
$\V_{k+1} = \V_k \cup \set{ \Eop\alpha(S) \mid \alpha\in\atoms, S\subseteq\V_k}$,
the vocabulary $\V$ is the \textbf{closure} of $\bigcup_{k\geq 0} \V_k$, when all
these sets are defined; otherwise, $\V$ is undefined.
When defined, the set $\V$ is closed under $\Eop\alpha(S)$ for \textbf{finite}
subsets $S\subseteq\V$, and by taking limits since $\V$ is a (topologically) closed set.

Below there are conditions for a net to be regular with well-defined
vocabulary. In such a case, $\Phi$ factors the space of infinite types into
equivalence classes:
$\itw\sim\itw'$ iff $\Phi(\itw) = \Phi(\itw')$.
The quotient table $\Table: \atoms \times \powerset(\V) \to \V$
maps atomic types and \textbf{finite} subsets of embeddings into
embeddings as $(\alpha, S) \mapsto \Eop\alpha(S)$.
The embedding for pointed graph $(x,u)$ is
$\Phi(\itype(x,u))\,{=}\,\Table(\alpha, \Phi(S))$ for
$\alpha\,{=}\,\atp(u)$ and $S\,{=}\,\set{ \itype(x,v) \mid E(x,u,v) }$.
\emph{The table doesn't know about infinite types, only 
about atomic types and embeddings in the vocabulary.}

A crucial question is whether $\Phi(\itype(x,u))$ can be computed from
the table alone, without the need to construct types.

For fixed graph $x$ with $n$ vertices, we consider tensors
$\bm{E} \in \R^{nW}$ which contain embeddings $\bm{E}_u\in\Embeddings$
for each vertex $u$.
Let $\tup{\GEmb{k}}_{k\geq0}$ be the sequence
$\GEmbu0{u} \doteq \Eop\alpha(\emptyset)$ if $u$ is of atomic type $\alpha$, and
$\GEmbu{k+1}{u} \doteq \Table( \alpha, \set{ \GEmbu{k}{v} \mid E(x,u,v) } )$.
The sequence is well defined and $\GEmbu{k}{u} \in \V$
for each vertex $u$ and $k\geq 0$.
Does $\GEmbu{k}{u}\to\Phi(\itype(x,u))$ as $k\to\infty$?

If the vocabulary $\V$ is finite, the sequence $\tup{\GEmb{k}}_{k\geq0}$
is the \emph{trace of a cellular automaton} \cite{wolfram2002new} on
the graph $x$, where each vertex starts at a configuration determined by
its atomic type, and changes its configuration according to
the configuration of its neighbors as dictated by the table.

\begin{definition}[Dynamical systems]
  \label{def:game}
  Let $\Net(\cdot, \cdot, \cdot)$ be a regular net with vocabulary $\V$.
  \begin{enumerate}[leftmargin=*, label=$\bullet$] 
    \item 
      For graph $x$ with $n$ vertices, the \textbf{game} induced by $\Table$ on $x$ is the
      \emph{dynamical system} on $\V^n$:
      \[ \GEmbu0{u} = \Eop\alpha(\emptyset)\,, \  \GEmbu{k+1}{u} = \Table(\alpha, \set{ \GEmbu{k}{v} \mid E(x,u,v) }) \,. \]
      where $\alpha=\atp(u)$ is the atomic type for vertex $u$.
    \item 
      The game is \textbf{admissible for graph $x$} iff
      $\lim_{k\to\infty} \GEmbu{k}{u} = \Phi(\itype(x,u))$ for every vertex $u\in V(x)$,
      and it is \textbf{admissible} iff it is admissible for every graph.
    \item 
      The game is \textbf{ordered} iff there is an order $(\V,\preceq)$
      such that $\GEmbu{k}{u} \preceq \GEmbu{k+1}{u}$ for every graph $x$,
      vertex $u$ in $x$, and $k\geq 0$.
    \item 
      The game is \textbf{monotone} iff, writing $S \sqsubseteq S'$ for
      finite $S, S' \subseteq\V$ to mean every $\bm{e}\in S$ is
      $\preceq$-dominated by some $\bm{e}'\in S'$, the table satisfies
      $\T(\alpha,S) \preceq \T(\alpha,S')$ whenever $S \sqsubseteq S'$.
  \end{enumerate}
\end{definition}

For admissible tables over finite vocabulary, running the dynamical
system is equivalent to running the net. As the states of the dynamical
system correspond to the net's final embeddings, the dynamical system
may converge faster than the net.
For a graph of order $n$, an admissible dynamical system converges
within $|\V|^n$ steps, or $n(|\V|-1)$ steps if the game is ordered.
Contractive nets converge in the limit. 

A converging regular net with finite vocabulary whose game is
admissible, ordered and monotone is called a \ratchet net.
Corollary~\ref{corollary} shows that for
$\oplus\,{=}\,\max$, the induced game is monotone.
If the vocabulary is totally ordered by $\leq$, the size of
the table reduces to $\mathcal{O}(|\atoms|\cdot(|\V|+1))$.


\subsection{Contractive Nets}

The original recurrent GNN is a contractive net \cite{gori:gnn}, and the
regime persists in modern equilibrium models: implicit GNNs \cite{gu2020implicit}
enforce a weight-norm contraction bound by projection during training,
guaranteeing convergence by construction — a checkable condition of the same species as Theorem A1.
The architecture $\Net(\cdot, \cdot, \cdot)$ over the atomic types $\atoms$ is \textbf{contractive
for $\|\!\cdot\!\|$,} if there is $\lambda\in(0,1)$ such that for each 
$\alpha$,
embeddings $z,z' \in \Embeddings$, and finite subsets $S, S' \subseteq \Embeddings$:
\begin{alignat*}{1}
  \| \Net(\alpha, z, S) &- \Net(\alpha, z', S') \| \\
  &\qquad \leq \lambda \max ( \|z - z'\|, \|\oplus\! S - \oplus S'\|) \,,
\end{alignat*}
whereas it is \textbf{contractive in $z$} if is contractive when $S=S'$.
The aggregator $\oplus$ is \textbf{1-Lipschitz for $\|\!\cdot\!\|$,}
if for any two vertex-indexed sets $S\,{=}\,\set{ \bm{e}_v \mid v\,{\in}\,N }$
and $S'\,{=}\,\set{ \bm{e}'_v \mid v\,{\in}\,N }$:
$ \| \oplus\! S - \oplus S' \| \leq \max_{v \in N} \| \bm{e}_v - \bm{e}'_v \| $.
(Notice that $S$ and $S'$ may have different size even though both are indexed by $N$.)

\begin{theoremrep}[Contractive nets]
  \label{thm:contractive}
  Let $\Net(\cdot, \cdot, \cdot)$ be a recurrent net over the atomic types in $\atoms$,
  and let $\|\cdot\|$ be a norm in \Embeddings. Then:
  \begin{enumerate}[leftmargin=*, label=\arabic*.]
    \item If the net is contractive for $\|\cdot\|$ and the aggregator is 1-Lipschitz for
      $\|\cdot\|$, the net is converging.
    \item If the net is converging and \textbf{contractive in $z$} for $\|\cdot\|$,
      the net is regular and the vocabulary is well defined.
    \item If the net is contractive and the aggregator is 1-Lipschitz (both) for $\|\cdot\|$,
      the net is regular, the vocabulary is well defined, and the game is admissible.
  \end{enumerate}
\end{theoremrep}
\begin{proof}
  \textbf{Claim 1.}
  Fix a graph $x$ with $n$ vertices, consider tensors $\bm{z}\in\R^{nW}$ that assign
  embeddings $\bm{z}_u$ to vertices $u$ in the graph, and define the mapping
  $\Gamma: \R^{nW} \to \R^{nW}$ as $\Gamma(\bm{z})_u = \Net(\atp(u), \bm{z}, \set{ \bm{z}_v \mid E(x,u,v) })$.
  Lift the norm to such tensors as $|\!|\!|\bm{z}|\!|\!|=\max_{u\in V(x)}\|z_u\|$.

  \medskip
  We show that $\Gamma$ is a contraction for $|\!|\!|\cdot|\!|\!|$.
  Let $\bm{z}$ and $\bm{z}'$ be two tensors, and let
  $S_u = \set{ \bm{z}_v \mid v \in N(u) }$ and $S'_u = \set{ \bm{z}'_v \mid v \in N(u) }$ for
  vertex $u$ where $N(u) = \set{ v \mid E(x,u,v) }$ is $u$'s neighborhood. Then,
  \begin{alignat*}{1}
    |\!|\!| \Gamma(\bm{z}) - \Gamma(\bm{z}') |\!|\!|
      &=    \max_{u \in V(x)} \| \Gamma(\bm{z})_ u - \Gamma(\bm{z}')_ u \| \\
      &=    \max_{u \in V(x)} \| \Net(\atp(u), \bm{z}_u, S_u ) - \Net(\atp(u), \bm{z}'_u, S'_u ) \| \\
      &\leq \lambda \, \max_{u \in V(x)} \max ( \| \bm{z}_u - \bm{z}'_u \|, \|\oplus S_u - \oplus S'_u \| ) \\
      &\leq \lambda \, \max_{u \in V(x)} \max ( |\!|\!| \bm{z} - \bm{z}' |\!|\!|, \max_{v \in N(u)} \|\bm{z}_v - \bm{z}'_v \| ) \\
      &\leq \lambda \, \max_{u \in V(x)} \max ( |\!|\!| \bm{z} - \bm{z}' |\!|\!|, |\!|\!| \bm{z} - \bm{z}' |\!|\!| ) \\
      &\leq \lambda \, |\!|\!| \bm{z} - \bm{z}' |\!|\!|
  \end{alignat*}
  where the first inequality is by the contraction of the net, and the second since $\oplus$
  is 1-Lipschitz.
  By Banach's fixed-point theorem, there is a \textbf{unique fixed-point} $\bm{z}^*$
  satisfying $\Gamma(\bm{z}^*) = \bm{z}^*$, and the sequence $\bm{z}_{\ell+1}=\Gamma(\bm{z}_\ell)$
  converges to $\bm{z}^*$ from any starting point $\bm{z}_0$.
  Notice that we apply Banach's fixed-point theorem to fixed graph $x$, yet the contraction
  factor $\lambda$ is the same for every graph.

  \medskip\noindent\textbf{Claim 2.}
  We show two things: the net is regular and the vocabulary is well defined.
  Fix a pointed graph $(x,u)$ of atomic type $\alpha$, and let $\Phi_u$ denote $\Phi(\itype(x,u))$.

  \emph{The net is regular.}
  We need to show $\lim_{\ell\to\infty} z_\ell = \Phi_u$ where $z_0=\initial\alpha$ and $z_{\ell+1} = \Net(\alpha, z_\ell, S)$
  with $S = \set{ \Phi_v \mid E(x,u,v) }$:
  \begin{alignat*}{1}
    \| \Phi_u - z_{\ell+1} \|
      &=     \| \Net(\alpha, \Phi_u, S) - \Net(\alpha, z_{\ell}, S) \|
      \leq \lambda \, \| \Phi_u - z_{\ell} \|
       \leq \cdots \leq  \lambda^{\ell+1} \, \| \Phi_u - \initial\alpha \|
  \end{alignat*}
  which goes to $0$ as $\ell\to\infty$. Therefore, the limit exists and it is equal to $\Phi_u$.

  \emph{The vocabulary is well defined.}
  It is sufficient to show that $\Eop\alpha(S)$ exists for every finite subset $S$ of embeddings.
  We show that the sequence 
  $z_0=\initial\alpha$ and $z_{\ell+1} = \Net(\alpha, z_\ell, S)$ is Cauchy.
  This is direct since $z \mapsto  \Net(\alpha, z, S)$ is a contraction.

  \medskip\noindent
  \textbf{Claim 3.}
  By Claim 1 the net is converging. Contractivity implies Contractivity in $z$ (take $S=S'$),
  so by Claim 2 the net is regular and the vocabulary $\V$ is well defined.
  Hence, $\T$ and the game are defined.
  It remains to show admissibility; i.e., $\GEmbu{k}{u} \to \Phi_u$ as $k\to\infty$
  for every pointed graph $(x,u)$.

  \medskip
  Fix graph $x$ and let $\delta_k = \max_{u\in V(x)} \| \GEmbu{k}{u} - \Phi_u \|$, which is
  finite since $x$ is a finite graph.
  Let $\alpha$ be the atomic type of $u$, let $N=N(u)$ be $u$'s neighborhood, and
  set $S^k = \set{ \GEmbu{k}{v} \mid v \in N}$ and $S^* =\set{ \Phi_v \mid v \in N}$.
  By regularity, $\Phi_u = \Eop\alpha(S^*)$, and by the definition of
  the game $\GEmbu{k+1}{u} = \T(\alpha,S^k) = \Eop\alpha(S^k)$.

  If $N=\emptyset$ then $S^k=S^*=\emptyset$, so $\GEmbu{k+1}{u} = \Eop\alpha(\emptyset) = \Phi_u$
  and $\|\GEmbu{k+1}{u} - \Phi_u\| = 0 \leq \lambda\delta_k$.
  Otherwise, since $\Eop\alpha(S)$ is the unique fixed point of $z\mapsto\Net(\alpha,z,S)$ for
  every finite $S$:
  \begin{alignat*}{1}
    \|\Eop\alpha(S^k) - \Eop\alpha(S^*)\|
         =    \| \Net(\alpha,\Eop\alpha(S^k),S^k) - \Net(\alpha,\Eop\alpha(S^*),S^*)\|
         \leq \lambda \max \bigl( \| \Eop\alpha(S^k) - \Eop\alpha(S^*) \|, \|\oplus S^k - \oplus S^*\| \bigr) \,.
  \end{alignat*}
  If the maximum is attained by the first argument,
  $\|\Eop\alpha(S^k) - \Eop\alpha(S^*)\| \leq \lambda \|\Eop\alpha(S^k) - \Eop\alpha(S^*)\|$
  with $\lambda<1$, forcing $\|\GEmbu{k+1}{u} - \Eop\alpha(S^*)\|=0 \leq \lambda\delta_k$.
  If it is attained by the second argument, 1-Lipschitz gives
  \[ \|\GEmbu{k+1}{u} - \Phi_u\| \leq \lambda\,\|\oplus S^k - \oplus S^*\| \leq \lambda \max_{v \in N} \|\GEmbu{k}{v} - \Phi_v\| \leq \lambda\delta_k \,. \]
  In every case $\|\GEmbu{k+1}{u} - \Phi_u\| \leq \lambda \delta_k$, so $\delta_{k+1}\leq\lambda\delta_k$
  and therefore $\delta_k\leq\lambda^k\delta_0$ with $\delta_0=\max_{v\in V(x)}\|\Eop{\atp(v)}(\emptyset) - \Phi_v\|$.
  Since $\lambda<1$, $\delta_k\to 0$ as $k\to\infty$.
  That is, $\GEmbu{k}{u}\to\Phi_u$ for every vertex $u$, and the game is admissible.
  Moreover, since $\lambda$ does not depend on $x$, the convergence is uniform over all graphs.
\end{proof}

Nets that satisfy the conditions of Theorem~\ref{thm:contractive}
are those with max aggregator and tensors such that
$\|\bm{A}_\alpha\|_\infty < 1$ (contractive in $z$ for $L_\infty$), and
$\|\bm{A}_\alpha\|_\infty + \|\bm{B}_\alpha\|_\infty < 1$
(fully contractive for $L_\infty$) (see Theorem~\ref{thm:contractive:sufficient} in appendix).

\begin{toappendix}
  \begin{apxtheorem}[Sufficient Conditions for Contractive nets]
    \label{thm:contractive:sufficient}
    Let $\Net(\cdot, \cdot, \cdot)$ be a recurrent net for the atomic types in $\atoms$.
    \begin{enumerate}[leftmargin=*, label=\arabic*.]
      \item If the aggregator $\oplus$ is the point-wise \textbf{maximum}, the aggregator is 1-Lipschitz for $\|\cdot\|_\infty$.
      \item Let $\rho$ be a 1-Lipschitz for $\|\cdot\|_\infty$ activation function (e.g., tanh, sigmoid, ReLU or any clipping
        $x \mapsto \min(b, \max(a, x))$), and consider the \emph{maximum absolute row sum norm} $\|\cdot\|_{mars}$ over matrices.
        If $\|\bm{A}_\alpha\|_{mars} < 1$, the net is a contraction in $z$ for $\|\cdot\|_\infty$.
        If $\|\bm{A}_\alpha\|_{mars} + \|\bm{B}_\alpha\|_{mars} < 1$, the net is a contraction for $\|\cdot\|_\infty$.
      \Omit{
      \item If $\|\bm{A}_\alpha\|_\infty < 1$, the net is a contraction in $z$ for $\|\cdot\|_\infty$.
      \item If $\lambda = \|\bm{A}_\alpha\|_\infty + \|\bm{B}_\alpha\|_\infty < 1$ (where the norm here is
        the row-sum or operator norm) and the activation function is 1-Lipschitz (e.g., tanh, sigmoid, ReLU or any clipping
        $x \mapsto \min(b, \max(a, x))$), the net is contractive for $\|\cdot\|_\infty$.
      }
    \end{enumerate}
  \end{apxtheorem}
  \begin{proof}
    \textbf{Claim1.}
    Let $(x,u)$ be a pointed graph, let $N$ be a subset with $n$ vertices, and let $S$ and $S'$
    be two sets of embeddings indexed by the vertices in $N$: the embeddings in $S$ (resp.\ $S'$) are
    denoted by $\bm{e}_v$ (resp.\ $\bm{e}'_v$) for $v \in N$.
    We need to show that the max aggregator is 1-Lipschitz for $\|\cdot\|_\infty$.

    For register $i$, assume without loss of generality $(\oplus S)_i \geq (\oplus S')_i$:
    \[ | \max_{v \in N} \bm{e}_{v,i} - \max_{v \in N} \bm{e}'_{v,i} | = \max_{v \in N} \bm{e}_{v,i} - \max_{v \in N} \bm{e}'_{v,i} \leq \bm{e}_{v^*,i} - \bm{e}'_{v^*,i} \leq \max_{v \in N} |\bm{e}_{v,i} - \bm{e}'_{v,i}| \]
    where $v^* \in N$ is the index such that $\bm{e}_{v^*,i} = \max_v \bm{e}_{v,i}$. Hence,
    \[ \| \oplus\!S - \oplus S' \|_\infty = \max_i | \max_{v \in N} \bm{e}_{v,i} - \max_{v \in N} \bm{e}'_{v,i} | \leq \max_i \max_{v \in N} | \bm{e}_{v,i} - \bm{e}'_{v,i} | = \max_{v \in N} \| \bm{e}_v - \bm{e}'_v \|_\infty \,. \]

    \medskip\noindent
    \textbf{Claim 2.}
    Let $\alpha$ be an atomic type, let $z$ and $z'$ be two embeddings, and let $S$ and $S'$ be
    two vertex-indexed finite sets of embeddings. Then, for $\lambda=\|\bm{A}_\alpha\|_{mars} + \|\bm{B}_\alpha\|_{mars}$:
    \begin{alignat*}{1}
      \| \Net(\alpha, z, S) - \Net(\alpha, z', S') \|_\infty  \
      &=     \| \rho( \bm{A}_\alpha\,z + \bm{B}_\alpha(\oplus S) + \bm{C}_\alpha  ) - \rho( \bm{A}_\alpha\,z' + \bm{B}_\alpha(\oplus S') + \bm{C}_\alpha  ) \|_\infty \\
      &\leq  \| ( \bm{A}_\alpha\,z + \bm{B}_\alpha(\oplus S) + \bm{C}_\alpha  ) - ( \bm{A}_\alpha\,z' + \bm{B}_\alpha(\oplus S') + \bm{C}_\alpha  ) \|_\infty \\
      &=     \| \bm{A}_\alpha\,(z - z') + \bm{B}_\alpha(\oplus S - \oplus S') \|_\infty \\
      &\leq  \| \bm{A}_\alpha\,(z - z') \|_\infty + \| \bm{B}_\alpha(\oplus S - \oplus S') \|_\infty \\
      &\leq  \|\bm{A}_\alpha\|_{mars} \, \|z - z'\|_\infty + \|\bm{B}_\alpha\|_{mars} \, \|\oplus S - \oplus S'\|_\infty \\
      &\leq  \lambda \, \max( \|z - z'\|_\infty, \|\oplus S - \oplus S'\|_\infty )
    \end{alignat*}
    where the first inequality is from 1-Lipschitz of activation, the second by triangle inequality,
    the third by properties of the mars norm (it is the induced operator norm for $\|\cdot\|_\infty$),
    and the last by the definition of $\lambda$.
    If $S=S'$, the term $ \|\oplus S - \oplus S'\|_\infty = 0$.
    Hence, if $\|\bm{A}_\alpha\|_{mars} < 1$, the net is a contraction for $z$ for  $\|\cdot\|_\infty$,
    and the net is a (full) contraction for  $\|\cdot\|_\infty$ if $\lambda<1$.
  \end{proof}
\end{toappendix}

\subsection{Monotone Nets}

The architecture $\Net(\cdot, \cdot, \cdot)$ 
is \textbf{monotone} if for each $\alpha$, embeddings $z,z' \in \Embeddings$,
and finite subsets $S, S' \subseteq \Embeddings$:
\[ z \leq z' \ \land\  \oplus S \leq \oplus S' \implies \Net(\alpha, z, S) \leq \Net(\alpha, z', S') \,. \]
The aggregator $\oplus$ is \textbf{monotone} if for any two vertex-indexed embedding
sets $S=\set{ \bm{e}_v \mid v \in N }$ and $S'=\set{ \bm{e}'_v \mid v \in N }$:
$ \forall(v \in N)[ \bm{e}_v \leq \bm{e}'_v ] \implies \oplus S \leq \oplus S' $.

\begin{theoremrep}[Monotone nets]
  \label{thm:monotone}
  Let $\Net(\cdot, \cdot, \cdot)$ be a converging net over the atomic types in $\atoms$.
  Assume the initial embeddings are non-negative pre-fixed points of the isolated-vertex
  update: $0\leq\initial\alpha\leq\Net(\alpha,\initial\alpha,\emptyset)$ for $\alpha\,{\in}\,\atoms$.
  Then:
  \begin{enumerate}[leftmargin=*, label=\arabic*.]
    \item If the net is monotone, the vocabulary is well defined.
    \item If the net and the aggregator are both monotone, the net is regular, and the game is admissible and ordered.
  \end{enumerate}
\end{theoremrep}
\begin{proof}
  Throughout $\leq$ denotes the component-wise partial order on \Embeddings.
  Aggregator monotonicity is used in the indexed form: if $S = \set{\initial{v}  \mid v\in N}$
  and $S' = \set{ \initial{v}' \mid v\in N}$ with $\initial{v}\leq\initial{v}'$ for every $v\in N$,
  then $\oplus S \leq \oplus S'$.
  We also use $\oplus\emptyset = 0 \leq \oplus S$ for finite sets $S$ of non-negative embeddings.

  \medskip\noindent
  \textbf{Claim 1 (vocabulary).}
  We show by induction that for every $\alpha$ and finite set $S$ of non-negative
  embeddings, the canonical sequence $z_0=\initial{\alpha}$ and $z_{\ell+1} = \Net(\alpha, z_\ell, S)$
  is $\leq$-non-decreasing, its limit $\Eop\alpha(S)$ exists, and $\Eop\alpha(S)$
  is the least fixed point of $z \mapsto \Net(\alpha, z, S)$ above $\initial{\alpha}$.
  Hence, the vocabulary is well defined and consists of non-negative embeddings.

  First, we show by induction on $\ell$: $z_\ell \leq z_{\ell+1}$.
  The base case is
  \[ z_0 = \initial\alpha \leq \Net(\alpha, \initial\alpha, \emptyset) \leq \Net(\alpha, \initial\alpha, S) = z_1 \]
  since $\initial\alpha$ is pre-fixed point of $z \mapsto \Net(\alpha, z, \emptyset)$,
  and $\oplus\emptyset = 0 \leq \oplus S$.
  The inductive step:
  \[ z_{\ell} = \Net(\alpha, z_{\ell-1}, S) \leq \Net(\alpha, z_\ell, S) = z_{\ell+1} \]
  where the inequality follows by the inductive hypothesis $z_{\ell-1} \leq z_\ell$.
  Therefore, the limit $\Eop\alpha(S) = \lim_{\ell\to\infty} z_\ell$ exists.
  Indeed, $\Eop\alpha(S)$ is the \textbf{least fixed-point} of $z \mapsto \Net(\alpha, z, S)$
  that is above $\initial\alpha$.
  (Notice the third argument for the net is ``frozen'' at $S$ requiring no monotonicity
  of the aggregator).

  $\V_0$ is the set of initial embeddings which are all non-negative.
  Thus, $\V_1$ is well-defined and consists of non-negative embeddings as well.
  Inductively, the vocabulary $\V$ which is the closure of $\bigcup_{k\geq 0} \V_k$
  is well defined, and consists of non-negative embeddings.

  \begin{innerlemma}[Monotonicity of the table]
    \label{lemma:monotone}
    For every $\alpha\in\atoms$ and finite sets $S, S'$ of non-negative embeddings.
    If $\oplus S \leq \oplus S'$, then $\Eop\alpha(S) \leq \Eop\alpha(S')$.
  \end{innerlemma}
  \noindent\emph{Proof of Lemma:}
    Let $\tup{z_\ell}_\ell$ and $\tup{z'_\ell}_\ell$ be the canonical sequences for
    $(\alpha,S)$ and $(\alpha,S')$, respectively.
    First, we show by induction on $\ell$: $z_\ell \leq z'_\ell$ for every $\ell\geq 0$.
    The base case is the equality $z_0 = z'_0 = \initial\alpha$.
    By inductive hypothesis $z_\ell \leq z'_\ell$.
    Therefore, $z_{\ell+1} = \Net(\alpha, z_\ell, S) \leq \Net(\alpha, z'_\ell, S') = z'_{\ell+1}$
    by the monotonicity of the net since $\oplus S \leq \oplus S'$.
    Second, pass to the limit to get $\Eop\alpha(S) \leq \Eop\alpha(S')$.
    In particular, $\initial\alpha \leq \Eop\alpha(\emptyset) \leq \Eop\alpha(S)$ for every such $S$.
    \hfill\ensuremath{\scaleobj{.75}{\blacksquare}}    

  \medskip\noindent
  \textbf{Claim 2 (regularity, admissibility, order).}
  Fix a graph $x$, let $u\in V(x)$ be a vertex of atomic type $\alpha=\atp(u)$, and write
  \begin{enumerate}[leftmargin=*, label=--]
    \item $z_{u,\ell}$ for the net's synchronous run: $z_{u,0} = \initial\alpha$,
      and $z_{u,\ell+1} = \Net(\alpha, z_{u,\ell}, \set{z_{v,\ell} \mid E(x,u,v)})$,
    \item $\Phi_u = \lim_{\ell\to\infty} z_{u,\ell}$ for its limit,
    \item $S^*_u = \set{ \Phi_v \mid E(x,u,v)}$ for the final embeddings of $u$'s neighbors, and
    \item $\GEmbu{k}{u}$ for the game's run: $\GEmbu{0}{u} = \Eop\alpha(\emptyset)$,
      and $\GEmbu{k+1}{u} = \T(\alpha, S^k_u) = \Eop\alpha(S^k_u)$ with
      $S^k_u = \set{ \GEmbu{k}{v} \mid E(x,u,v)}$.
  \end{enumerate}

  \smallskip\noindent
  \textbf{Step \emph{2a} (regularity).}
  By definition, the net is regular iff $\Phi_u = \Eop\alpha(S^*_u)$.
  We show $z_{u,\ell} \leq \Eop\alpha(S^*_u)$ for every vertex $u$ and layer $\ell$.
  Indeed, suppose that this is not true, and let $\ell\geq0$ be a \textbf{minimum}
  layer index and $u$ a vertex such that $z_{u,\ell} \not\leq \Eop\alpha(S^*_u)$.
  Clearly, $\ell>0$ since $z_{u,0}=\initial\alpha \leq \Eop\alpha(S^*_u)$ by Lemma~\ref{lemma:monotone}.
  Since $\ell>0$ is minimum, $z_{u,\ell-1} \leq \Eop\alpha(S^*_u)$ and
  $z_{v,\ell-1} \leq \Eop{\atp(v)}(S^*_v)$ for neighbor $v$ of $u$.
  The monotonicity of the aggregator and inductive hypothesis gives
  $\oplus S^{\ell-1}_u \leq \oplus S^*_u$.
  Apply the monotonicity of the net to get
  \begin{alignat*}{1}
    z_{u,\ell} =    \Net( \alpha, z_{u,\ell-1}, S^{\ell-1}_u )
               \leq \Net( \alpha, \Eop\alpha(S^*_u), S^*_u )
               =    \Eop\alpha(S^*_u) \,,
  \end{alignat*}
  where the last equality is because $\Eop\alpha(S^*_u)$ is a fixed point of
  $z\mapsto\Net(\alpha,z,S^*_u)$.
  This contradicts the supposition.
  Hence, $z_{u,\ell} \leq \Eop\alpha(S^*_u)$ for every vertex $u$ and layer $\ell$.
  Passing to the limit (as the net is a continuous function)
  yields $\Phi_u \leq \Eop\alpha(S^*_u)$.

  Since $0\leq\initial\alpha$ for $\alpha\in\atoms$, $0 \leq \oplus S^0_u$ for each vertex $u$.
  The pre-fixed point hypothesis and monotonicity of the net give
  $z_{u,0} = \initial\alpha \leq \Net(\alpha,\initial\alpha,\emptyset) \leq \Net(\alpha,\initial\alpha,S^0_u) = z_{u,1}$,
  and inductively $z_{u,\ell} \leq z_{u,\ell+1}$.
  Hence $\initial\alpha = z_{u,0} \leq z_{u,1} \leq \cdots \leq \Phi_u \leq \Eop\alpha(S^*_u)$.
  Since $\Phi_u$ is also a fixed point of $z \mapsto \Net( \alpha, z, S^*_u)$ above $\initial\alpha$,
  and $\Eop\alpha(S^*_u)$ is the least fixed point above $\initial\alpha$, then
  $\Phi_u = \Eop\alpha(S^*_u)$.
  Hence, \emph{the net is regular.}

  \medskip\noindent
  \textbf{Step \emph{2b}: the game's run is $\leq$-non-decreasing.}
  We show by induction on $k$, $\GEmbu{k}{u} \leq \GEmbu{k+1}{u}$ for every vertex $u$.
  Base case: $\GEmbu{0}{u} = \Eop\alpha(\emptyset) \leq \Eop\alpha(S^0_u) = \GEmbu{1}{u}$
  by Lemma~\ref{lemma:monotone}, since $\oplus\emptyset = 0 \leq \oplus S^0_u$.
  Inductive step: if $\GEmbu{k-1}{v} \leq \GEmbu{k}{v}$ for every vertex $v$,
  then $\oplus S^{k-1}_u \leq \oplus S^k_u$ by the monotonicity of the aggregator
  (both sets indexed by $N(u)$).
  Hence, $\GEmbu{k}{u} = \Eop\alpha(S^{k-1}_u) \leq \Eop\alpha(S^k_u) = \GEmbu{k+1}{u}$
  by Lemma~\ref{lemma:monotone}.

  \medskip\noindent
  \textbf{Step \emph{2c}: admissibility.}
  Two bounds sandwich the run of the game.
  \begin{enumerate}[leftmargin=*, label=$(\roman*)$]
    \item $\GEmbu{k}{u} \leq \Phi_u$ for every $k$, by induction on $k$.
      The base is $\GEmbu{0}{u} = \Eop\alpha(\emptyset) \leq \Eop\alpha(S^*_u) = \Phi_u$
      directly by Lemma~\ref{lemma:monotone} and regularity, since $\oplus\emptyset = 0 \leq \oplus S^*_u$.
      If $\GEmbu{k}{v} \leq \Phi_v$ for every $v$, then $\oplus S^k_u \leq \oplus S^*_u$,
      so $\GEmbu{k+1}{u} = \Eop\alpha(S^k_u) \leq \Eop\alpha(S^*_u) = \Phi_u$ by
      Lemma~\ref{lemma:monotone} and regularity.
    \item $z_{u,\ell} \leq \GEmbu{k}{u}$ for every $\ell$, by induction on $k$.
      The base is $z_{u,0}=\initial\alpha\leq\Eop\alpha(\emptyset)=\GEmbu{0}{u}$ since the canonical sequence for $(\alpha,\emptyset)$ is non-decreasing from $\initial\alpha$.
      For the inductive step, note that $\GEmbu{\ell}{u} \leq \Eop\alpha(S^{\ell}_u)$ for every $\ell$
      (for $\ell=0$ by Lemma~\ref{lemma:monotone}'s final remark; for $\ell\geq1$ because
      $\GEmbu{\ell}{u} = \Eop\alpha(S^{\ell-1}_u) \leq \Eop\alpha(S^{\ell}_u)$, as in Step 2b).
      Then,
      $z_{u,\ell+1} = \Net(\alpha, z_{u,\ell}, \set{z_{v,\ell} \mid E(u,v)}) \leq
      \Net(\alpha, \Eop\alpha(S^{\ell}_u), S^{\ell}_u) = \Eop\alpha(S^\ell_u) = \GEmbu{\ell+1}{u}$,
      using, in order: the inductive hypothesis with net and aggregator monotonicity;
      $\GEmbu{\ell}{u} \leq \Eop\alpha(S^{\ell}_u)$ with net monotonicity; and the
      fixed-point property of $\Eop\alpha(S^{\ell}_u)$.
  \end{enumerate}

  \smallskip\noindent
  By Steps 2b and $(i)$, $\tup{\GEmbu{k}{u}}_k$ is component-wise non-decreasing and bounded
  above by $\Phi_u$, hence converges to some $\GEmbu{*}{u} \leq \Phi_u$.
  By $(ii)$, passing to the limit, $\Phi_u \leq \GEmbu{*}{u}$.
  Hence $\lim_{k\to\infty} \GEmbu{k}{u} = \Phi_u$ for every vertex of every graph;
  that is, the game is admissible.

  \medskip\noindent
  \textbf{Step \emph{2d}: order.}
  By Step 2b, the restriction of the component-wise order $\leq$ to $\V$ witnesses
  Definition~\ref{def:game}'s condition: take $\leq$ equal to the restriction of $\leq$
  to $\V$. Definition~\ref{def:game} requires an order, not a total one, so no extension
  is needed.
\end{proof}

Any fixed point is a pre-fixed point, and the all-zero initialization satisfies the hypothesis
automatically for ReLU or clipped activations,
since $\Net(\alpha, 0, \emptyset) = \rho(\bm{C}_\alpha) \geq 0$.

\begin{corollaryrep}
  \label{corollary}
  Under the hypotheses of Theorem~\ref{thm:monotone} with $\oplus\,{=}\,\max$,
  the game is monotone.
  Hence $\T(\alpha,S) = \Eop\alpha(S) \leq \Eop\alpha(S') = \T(\alpha,S')$
  whenever $S \sqsubseteq S'$.
\end{corollaryrep}
\begin{proof}
  By Theorem~\ref{thm:monotone}, the vocabulary $\V$ consists of non-negative
  embeddings and the game order $\preceq$ is the restriction of the component-wise
  order $\leq$ to $\V$ (Step 2d in proof of the theorem).
  Let $S \sqsubseteq S'$; i.e., every $\bm{e} \in S$ is $\leq$-dominated by some $\bm{e}' \in S'$.
  If $S = \emptyset$ then $\oplus S=0 \leq \oplus S'$, since all embeddings in $\V$ are non-negative.
  Otherwise, fix a register $i$. For every $\bm{e}\in S$ there is $\bm{e}' \in S'$ with $\bm{e}\leq\bm{e}'$,
  hence $\bm{e}_i \leq \bm{e}'_i \leq \max\set{\bm{e}''_i \mid \bm{e}'' \in S'} = (\oplus S')_i$.
  Taking the maximum over $\bm{e} \in S$ gives $(\oplus S)_i \leq (\oplus S')_i$, and since $i$
  is arbitrary, $\oplus S \leq \oplus S'$.
  By Lemma~\ref{lemma:monotone}, $\T(\alpha,S) = \Eop\alpha(S) \leq \Eop\alpha(S') = \T(\alpha, S')$.
\end{proof}

Nets that satisfy conditions of Theorem~\ref{thm:monotone}
are those with max aggregator and matrices $\bm{A}$ and $\bm{B}$
with non-negative weights (Theorem~\ref{thm:monotone:sufficient} in appendix).
If all weights are non-negative integers, the activation function
is clipped, and the initial embeddings are integer valued, the vocabulary is finite.

\begin{toappendix}
  \begin{apxtheorem}[Sufficient conditions for monotone nets]
    \label{thm:monotone:sufficient}
    Let $\Net(\cdot,\cdot,\cdot)$ be a recurrent architecture over the atomic types in $\atoms$
    with an aggregator that is 1-Lipschitz.
    \begin{enumerate}[leftmargin=*, label=\arabic*.]
      \item If $\bm{A}_\alpha,\bm{B}_\alpha \geq 0$ (i.e., non-negative weights), the net is monotone for $\|\!\cdot\!\|_\infty$.
      \item If $\bm{A}_\alpha,\bm{B}_\alpha$ have non-negative integer weights, the net has clipped activation, and the initial embeddings are integer valued, the vocabulary is finite.
    \end{enumerate}
  \end{apxtheorem}
  \begin{proof}
    \textbf{Claim 1.}
    Let $z$ and $z'$ be embeddings, and $S$ and $S'$ be finite sets of embeddings.
    If $z \leq z'$ and $\oplus S \leq \oplus S'$, for every $\alpha$:
    \begin{alignat*}{1}
      \bm{A}_\alpha \, z + \bm{B}_\alpha(\oplus S) + \bm{C}_\alpha \ &\leq \ \bm{A}_\alpha \, z' + \bm{B}_\alpha(\oplus S') + \bm{C}_\alpha \\
      \rho\bigl( \bm{A}_\alpha \, z + \bm{B}_\alpha(\oplus S) + \bm{C}_\alpha \bigr) \ &\leq \ \rho\bigl( \bm{A}_\alpha \, z' + \bm{B}_\alpha(\oplus S') + \bm{C}_\alpha \bigr) \\
      \Net( \alpha, z, S ) \ &\leq \ \Net( \alpha, z', S' ) \,.
    \end{alignat*}

    \medskip\noindent
    \textbf{Claim 2.}
    Under the assumptions, the embeddings in the vocabulary are integer valued and bounded. Therefore, there is a finite number of them.
  \end{proof}

\end{toappendix}

The order 
is what turns the dynamical system from
an unconstrained automaton into a stable, terminating one: forcing each vertex
forward through a fixed ladder of states rules out cycles and oscillation.
\Omit{
  The argument is in the spirit of self-stabilization in distributed systems
  \cite{dijkstra1974:self,altisen2019:book} and of well-quasi-ordering arguments
  certifying termination of monotone processes in verification \cite{finkel2001};
  the order plays the role of a Lyapunov function \cite{khalil2002nonlinear,chang2019}.
}
For finite vocabulary, the system is the trace of a cellular
automaton over the graph, with local rule given by table $\T$ \cite{grattarolaetal2021}.

Why compile the game rather than the run itself? Evaluating a fixed-point formula
on a finite graph is an iteration over tuples of vertex sets that stabilizes in
$\mathcal{O}(|V(x)|)$ rounds; the net's run instead traverses a possibly infinite
set of transient embeddings, possibly over infinite time.
A finitary formula can therefore track the net only through a finite quotient of
its state space, and the game is exactly that quotient with its induced dynamics.
The  lim-inf acceptance of \citet{pflueger2024:recurrent} and the
$\omega$GML characterization of \cite{Ahvonenetal2024} are trajectory-dependent
semantics, but they are infinitary or non-effective. 
At convergence the limits satisfy $\Phi_u = \T(\alpha, \set{\Phi_v \mid E(u,v)})$,
a system that in general has several solutions, of which the transients select one;
fixed-point logic can only name the extremal, iteratively generated solutions.
Admissibility says the dynamics selects the solution iteration does.

\newcommand{\row}{\mathrm{row}}

\section{Translating Recurrent Nets into IFP}
\nosectionappendix
\begin{toappendix}
  \myappendixheader{Appendix for section: Translating Recurrent Nets into IFP}
\end{toappendix}

We work with \emph{inflationary fixed-point logic} (IFP) over a relational
signature $\sigma$. Over finite structures (our target), IFP is
equivalent to least fixed-point logic (LFP) 
\cite{GurevichShelah1986:lfp-ifp}, but it has a simpler syntax and semantics.

For formula $\psi(R; \bar z)$ in $\FO[\sigma\cup\set{R}]$, where $k=|\bar{z}|$
and $R$ is a \emph{new} relation of arity $k$, the IFP constructor is
$\ifp{R,\bar{x}}{\psi}(\bar{t})$
where $\bar{t}$ is a tuple of terms of size $k$.
Such a formula holds iff $\bar{t}\in R^*$ where $R^*$ is
the inflationary fixed point of the one-step formula $\psi(R; \bar z)$;
see appendix.
IFP is extended to fixed points of multiple relations, simultaneously,
by considering vectors $\bm{R}\,{=}\,\tup{R_1,\ldots,R_m}$ of relations
defined by a vector $\psi(\bm(R); \bar z)\,{=}\,\tup{\psi_1,\ldots,\psi_m}$
of one-step relations.
Then,$\ifp{\bm{R}, \bar{x}}{\psi}_j(\bar t)$ iff $\bar{t}$ belongs to $R^*_j$.

\begin{theorem}[IFP Compiler]
  \label{thm:rec-compiler}
  Let $\Net(\cdot,\cdot,\cdot)$ be a regular net for the atomic
  types in $\atoms$ with finite vocabulary $\V\,{=}\,\set{\bm{e}_1,\ldots,\bm{e}_m}$,
  and whose game is admissible
  and ordered, but no game monotonicity required.
  There is a vector $\psi(\bm{R}; z)\,{=}\,\tup{\psi_1,\ldots,\psi_m}$
  of 1-step unary relations such that
  \[ \tup{\Psi_1,\ldots,\Psi_m}(u) \ \equiv \ \ifp{\bm{R}; z}{\psi}(u) \]
  characterizes the net; i.e.,
  for pointed graph $(x,u)$ and index $i$: 
  $\Phi(\tau(x,u))\,{=}\,\initial{i}$ iff $x\,{\vDash}\,\Psi_i(u) \land \bigwedge_{\initial{i}\,{\prec}\,\initial{j}}\! \neg\Psi_j(u)$.
\end{theorem}

\begin{toappendix}
  Let us be more precise about the syntax and semantics of IFP.
  For formula $\psi(R; \bar z)$ in $\FO[\sigma\cup\set{R}]$, where $k=|\bar{z}|$
  and $R$ is a \emph{new} relation of arity $k$, the IFP \emph{constructor} is
  \[ \ifp{R,\bar{x}}{\psi}(\bar{t}) \]
  where $\bar{t}$ is a tuple of terms of size $k$.
  The semantics is in terms of the \emph{inflationary fixed-point} $R^* = R^0 \cup R^1 \cup \cdots$
  where $R^0 \doteq \emptyset$ and $R^{i+1} \doteq R^i \cup \set{ \bar{z} \in U^k \mid \Au \vDash \psi(R_i; \bar{z}) }$,
  for a $\sigma$-structure $\Au$ with universe $U$.
  That is, at stage $i$, when constructing the set $R^{i+1}$, the relation $R$ inside $\psi$ is interpreted as $R^i$.
  Since $R^0 \subseteq R^1 \subseteq \cdots$, the sequence $\tup{R^i}_i$ converges in at most $U^k$ steps
  for any finite structure $\Au$.
  Then $\Au \vDash \ifp{R,\bar{x}}{\psi}(\bar{t})$ iff $\bar{t}\in R^*$.
  The syntax for IFP is simpler than the syntax for LFP as no requirements on the
  ``one-step'' formula $\psi$ are imposed.

  IFP is extended to fixed-points of multiple relations, simultaneously,
  by considering vectors $\bm{R}=\tup{R_1,R_2,\ldots,R_m}$ of relations
  instead of the single relation $R$.
  The denotation of each relation is initially empty, and they are extended (in parallel)
  with formulas $\psi_j(\bm{R}, \bar z)$ that may use the relations $R_1, R_2, \ldots, R_m$.
  The syntax for the inflationary fixed-point is $\ifp{\bm{R}, \bar{x}}{\psi}_j(\bar t)$
  and such a formula holds when the term $\bar{t}$ belongs to $R^*_j$.

  We construct an IFP formula to capture the convergence $\GEmb{\ell} \to \bm{E}^*$
  of an admissible and ordered game over a finite vocabulary $\V$.
  The idea is to construct unary relations for each index $i=1,2,\ldots,m$,
  where $m=|\V|$, to capture $\bm{E}^\ell_u=\initial{i}$.
  We use the order $\preceq$ on $\V$ obeyed by the game.

  Let $\Table : \atoms \times \powerset(\V) \to \V$ be the finite quotient table that
  describes the net, and consider the vector $\bm{R}=\tup{R_i}_{i=1,2,\ldots,m}$ of
  unary relations.
  The one-step formula $\psi(\bm{R}; u)=(\psi_1,\ldots,\psi_m)$ consists of $m$
  formulas $\psi^{start}_i \lor \psi^{next}_i$ that define the initial embeddings
  and their evolution according to $\Table$:
  \begin{alignat*}{1}
    \psi^{start}_i( u ) &\equiv \textstyle\bigvee_{\alpha{\in}\atoms, \Table(\alpha,\emptyset){=}\initial{i}} \alpha(u) \,, \\
    \psi^{next}_i(\bm{R}; u )  &\equiv
        \textstyle\bigvee_{\alpha{\in}\atoms, \emptyset{\subsetneq}S{\subseteq}\V, \Table(\alpha,S){=}\initial{i}} \row_{\alpha,S}(\bm{R}; u) \,,
  \end{alignat*}
  where the auxiliary formulas are:
  \begin{alignat*}{1}
    \row_{\alpha,S}(\bm{R} ; u)
        &\equiv \alpha(u) \land \forall v \biggl[ E(u,v) \Rightarrow \bigvee_{\initial{i}{\in}S} \curr{i}(\bm{R}; v) \biggr] 
        \land \bigwedge_{\initial{i}{\in}S} \exists v \bigl[ E(u,v) \land \curr{i}(\bm{R}; v) \bigr] \,, \\
    \curr{i}(\bm{R} ; v) &\equiv R_i(v) \land \textstyle\bigwedge_{\initial{i}{\prec}\initial{j}} \neg R_j(v) \,.
  \end{alignat*}

  \bigskip\noindent
  Theorem~\ref{thm:rec-compiler} in main text becomes the following
  Definition~\ref{def:rec-compiler:apx} and Theorem~\ref{thm:rec-compiler:apx}:

  \begin{apxdefinition}[IFP Compiler]
    \label{def:rec-compiler:apx}
    Let $\Net(\cdot,\cdot,\cdot)$ be a regular architecture for the atomic types in $\atoms$
    that have finite vocabulary $\V$, and admissible and ordered games.
    The \textbf{characterization} of the recurrent net is the vector
    $\tup{\Psi_1, \Psi_2, \ldots, \Psi_m}$ such that
    \[ \tup{\Psi_1, \Psi_2, \ldots, \Psi_m}(u) \ \equiv \ \ifp{\bm{R}; z}{\psi}(u) \,. \]
  \end{apxdefinition}

  \begin{apxtheorem}
    \label{thm:rec-compiler:apx}
    Let $\Net(\cdot,\cdot,\cdot)$ be a regular net for the atomic types in $\atoms$
    that has finite vocabulary $\V$ and whose game is admissible and ordered, and
    let $\tup{\Psi_1, \ldots, \Psi_m}$ be its characterization.
    For every pointed graph $(x,u)$ and $i=1,\ldots,m$,
    $\Phi(\tau(x,u)) = \initial{i}$ iff
    $x \vDash \Psi_i(u) \land \bigwedge_{\initial{i} \prec \initial{j}} \neg \Psi_j(u)$.
  \end{apxtheorem}
  \begin{proof}
    Let $x$ be a graph, let $\tup{\GEmb{k}}_{k\geq0}$ be the game run for graph $x$, and
    let $\tup{(R^k_1, R^k_2, \ldots, R^k_m)}_{k\geq 0}$ be the sequence of interpretations computed
    by the IFP semantics of $\ifp{\bm{R}; z}{\psi}$.

    \begin{innerlemma}
      \label{lemma:ifp}
      For $k\geq 0$, $u\in R^{k+1}_i$ iff $\GEmbu{l}{u} = \initial{i}$ for some $l\leq k$.
      (The shift on the index $k$ is due to the fact that $R^0_i=\emptyset$ whereas $\GEmbu{0}{u}$ is already an embedding.)
    \end{innerlemma}
    \noindent\emph{Proof of Lemma:}
    By induction of $k\geq0$:
    \begin{enumerate}[leftmargin=*, label=--]
      \item \textbf{Base case $k=0$.}

        \smallskip
        $(\Rightarrow)$. Assume $u\in R^1_i$. Hence, since $R^0_i$ is empty, $x \vDash \psi^{start}_i(u)$.
        Then, there is $\alpha$ such that $\atp(u)=\alpha$ and $\T(\alpha,\emptyset)=\initial{i}$
        implying $\GEmbu{0}{u} = \T(\atp(u),\emptyset) = \initial{i}$.

        \smallskip
        $(\Leftarrow)$. Assume $\GEmbu{0}{u} = \T(\atp(u),\emptyset) = \initial{i}$.
        Then, $x\vDash \psi^{start}_i(u) \implies u\in R^1_i$.

      \medskip
      \item \textbf{Inductive step $k>0$.}

        \smallskip
        $(\Rightarrow)$. Assume $u\in R^{k+1}_i$. Two cases: $u\in R^k_i$ or $x \vDash \psi^{next}_i(\bm{R}^k; u)$.
        In the first case, by inductive hypothesis, $\GEmbu{l}{u}=\initial{i}$ for some $l<k$.
        In the second case, by definition, there is $(\alpha,S)$ such that $\T(\alpha,S)=\initial{i}$ and $x \vDash \row_{\alpha,S}(\bm{R}^k; u)$.
        That is,
        \[ x \vDash \alpha(u) \land \forall v\biggl[ E(u,v) \implies \bigvee_{\initial{j}\in S} \curr{j}(\bm{R}^k; v) \biggr] \land \bigwedge_{\initial{j}\in S} \exists v \left[ E(u,v) \land \curr{j}(\bm{R}^k; v) \right] \]
        By definition, $\curr{j}(\bm{R}^k; v)$ holds iff $v \in R^k_j$ and $ v\notin R^k_{j'}$ for $\initial{j} \prec \initial{j'}$.
        Therefore, by inductive hypothesis, $\curr{j}(\bm{R}^k; v)$ holds iff $\GEmbu{l}{v}=\initial{j}$ for some $l<k$, and $\GEmbu{l}{v}\neq\initial{j'}$
        for all $l<k$ and $\initial{j}\prec\initial{j'}$. Since the game is ordered, $\curr{j}(\bm{R}^k; v)$ holds iff $\GEmbu{k-1}{v}=\initial{j}$.
        Hence, the vertex $u$ satisfies $\atp(u)=\alpha$, and $S = \set{ \GEmbu{k-1}{v} \mid E(x,u,v) }$ implying $\GEmbu{k}{u} = \T(\alpha,S) = \initial{i}$.

        \smallskip
        $(\Leftarrow)$. Assume $\GEmbu{l}{u} = \initial{i}$ for some $l\leq k$. We need to show $u\in R^{k+1}_i$.
        Two cases: $l < k$ or $l=k$.
        In the first case $l<k$, by inductive hypothesis, $u \in R^k_i$ which implies $u \in R^{k+1}_i$ since $R^{k+1}_i \supseteq R^k_i$.

        \smallskip
        Second case $l=k$.
        Since the game is ordered, $\GEmbu{k-1}{u} \prec \initial{i}$.
        Therefore, there is $(\alpha,S)$ such that $\atp(u)=\alpha$, $\T(\alpha,S)=\initial{i}$, and $S=\set{\GEmbu{k-1}{v} \mid E(x,u,v)}$.

        \smallskip
        Let $v$ be a vertex such that $E(x,u,v)$. By inductive hypothesis, $v\in R^k_j$ where $\GEmbu{k-1}{v}=\initial{j}$.
        On the other hand, since the game is ordered, $\GEmbu{l}{v}\neq\initial{j'}$ for $l<k-1$ and $\initial{j}\prec\initial{j'}$.
        Hence, $\curr{j}(\bm{R}^{k}; v)$ holds.
        Likewise, if $\initial{j}\in S$, there is vertex $v$ such that $E(x,u,v)$ and $\curr{j}(\bm{R}^{k}; v)$.
        Therefore, $x \vDash \psi^{next}_i(\bm{R}^{k}; u)$ which implies $u\in R^{k+1}_i$.
        \hfill\ensuremath{\scaleobj{.75}{\blacksquare}} 
    \end{enumerate}

    \begin{innerlemma}
      \label{lemma:ifp2}
      For $k\geq 0$, $\GEmbu{k}{u}=\initial{i} \iff x \vDash \curr{i}(\bm{R}^{k+1}; u)$.
    \end{innerlemma}
    \noindent\emph{Proof of Lemma:}
      If $\GEmbu{k}{u}=\initial{i}$, then $u \in R^{k+1}_i$ and $u \notin R^{k+1}_j$ for $\initial{i}\prec\initial{j}$
      since the game is ordered. Hence, $x \vDash \curr{i}(\bm{R}^{k+1}; u)$.
      Conversely, assume $x \vDash \curr{i}(\bm{R}^{k+1}; u)$.
      By definition, $u \in R^{k+1}_i$ and $u\notin R^{k+1}_j$ for $\initial{i}\prec\initial{j}$.
      By Lemma~\ref{lemma:ifp} and ordered game, $\GEmbu{k}{u}=\initial{i}$.
      \hfill\ensuremath{\scaleobj{.75}{\blacksquare}} 

    \bigskip\noindent
    At the fixed point, $x\vDash \Psi_i(u)$ iff $u \in R^*_i$.
    Hence, $x \vDash \Psi_i(u) \land \bigwedge_{\initial{i} \prec \initial{j}} \neg\Psi_j(u)$
    iff $\GEmbu{k}{u}$ has converged to $\initial{i}$ iff $\Phi(\itype(x,u)) = \initial{i}$
    (since the game is admissible).
  \end{proof}

  Definition~\ref{def:rec-compiler:apx} and Theorem~\ref{thm:rec-compiler:apx} require no monotonicity:
  the one-step formulas contain negation of the quantified predicates.
  This is the precise sense in which the IFP compiler is more general than the \Lmu compiler below
  that require the game to be monotone, but output \Lmu formulas.
  In exchange, the size of the IFP formula is linear in $|\T|$.
\end{toappendix}

\Omit{
  \subsection{Cascading Architectures}

  A cascading architecture is a sequence $\tup{\Net_k}_{k=0}^K$ of $K$ nets, each one defining an
  admissible and ordered game of finite vocabulary, and correspondences between the vocabulary of
  the net at depth $k$ and the atomic types at depth $k+1$.
  For an input graph, the first net $\Net_1$ yields embeddings for the graph vertices.
  The correspondence maps such embeddings into initial embeddings for the following net,
  which is run to yield new embeddings for the vertices. The process continues until the
  last net yields final embeddings for the vertices.

  Under the stated conditions, the compiler yields an IFP formula for each net in the sequence.
  These formulas are then composed to define an overall formula for the cascade. Such a formula
  is a $K$-nesting of IFP formulas: fixed points that are defined in terms of fixed points
  of lesser complexities. Since increasing the nesting depth of IFP formulas results in
  increased expressivity \cite{refs}, cascading architectures do increase expressive power of nets.
}

\subsection{Examples}

\Omit{
  We illustrate the IFP translation with the task of computing reachability
  of red vertices in graphs with black and red vertices.
  A second example considers a contractive net that generates an infinite number
  of transient embeddings that oscillates throughout convergence, and yet it has
  an admissible game over a finite and ordered vocabulary.
}

\newcommand{\red}{\mathrm{r}}
\newcommand{\green}{\mathrm{g}}
\newcommand{\black}{\mathrm{b}}

Reachability of red vertices in graphs with black and red vertices
can be solved with a net that implements a Boolean propagator:
vertex $u$ gets a mark if it is already marked, or some neighbor is marked.
The net has atomic types $\set{\black,\red}$, width $W\,{=}\,1$,
clipped activation, and initial embeddings, $\initial\black\,{=}\,0$
and $\initial\red\,{=}\,1$, that are fixed points of isolated vertices.
\Omit{
  The parameters $\bm{A}_\alpha = \bm{B}_\alpha = 1$ and $\bm{C}_\alpha=0$
  for $\alpha \in \set{\black,\red}$ yields:
  \begin{alignat*}{1}
    u^{(0)}        \; &= \; \bracket{\atp(u) = \red} \,, \\
    u^{(\ell+1)}   \; &= \; \clip\!\bigl( u^{(\ell)} + \max\set{ v^{(\ell)} \mid E(u,v) } \bigr) \,.
  \end{alignat*}
  which corresponds to the Boolean propagator:
  \[ u^{(\ell+1)}_x = 1 \;\iff\; u^{(\ell)}_x = 1 \;\text{ or }\; \exists v ( E(x,u,v) \land v^{(\ell)}_x = 1 ) \,. \]
}
By Theorem~\ref{thm:monotone} (and \ref{thm:monotone:sufficient} in appendix),
the net is regular, the game is admissible and ordered over 
$\V=\set{\initial\black, \initial\red}$, and the table is:
\begin{center}
  \begin{tabular}{lcc}
    \toprule
    Atomic type $\alpha$ & Neighborhood $S$ & $\Table(\alpha, S)$ \\
    \midrule
    $\mathrm{black}$ & $S\subseteq\set{0}$ & $0$ \\
    $\mathrm{black}$ & $1\in S$ & $1$ \\
    $\mathrm{red}$   & any & $1$ \\
  \end{tabular}
\end{center}
The compiler yields the one-step IFP formulas: 
\begin{alignat*}{1}
  R_\black(u) \ &\equiv\  \black(u) \land \forall v \bigl[ E(u,v) \implies R_\black(v) \land \neg R_\red(v) \bigr] \,, \\
  R_\red(u)   \ &\equiv\  \red(u) \lor \bigl( \black(u) \land \exists v \bigl[ E(u,v) \land R_\red(v) \bigr] \bigr)
\end{alignat*}
that match the standard one-step formulas for reachability.

\subsubsection{Contractive net.}
The same atomic types, width, and initial embeddings, but with parameters
$\bm{A}_\alpha\,{=}\,-\tfrac{1}{2}$ and $\bm{B}_\alpha\,{=}\,2$ for $\alpha\,{\in}\,\set{\black,\red}$,
$\bm{C}_\black\,{=}\,\tfrac{1}{2}$, and $\bm{C}_\red\,{=}\,1$.
The fixed-point embeddings for isolated black and red
vertices $\tfrac{1}{3}$ and $\tfrac{2}{3}$, resp., are
reached through infinite oscillating transient sequences.
%
%
The emergent vocabulary
$\V\,{=}\,\set{\initial1\,{=}\,\tfrac{1}{3}, \initial2\,{=}\,\tfrac{2}{3}, \initial3\,{=}\,\tfrac{7}{9}, \initial4\,{=}\,1}$
is finite and ordered, $\initial1\,{\prec}\,\initial2\,{\prec}\,\initial3\,{\prec}\,\initial4$,
with table:
\begin{center}
  \begin{tabular}{ccc}
    \toprule
    $\max S$       & $\Table(\black, S)$ & $\Table(\red, S)$ \\
    \midrule
    $0$            & $\initial1$ &  $\initial2$ \\
    $\initial1$    & $\initial3$ &  $\initial4$ \\
    $\initial2$    & $\initial4$ &  $\initial4$ \\
    $\initial3$    & $\initial4$ &  $\initial4$ \\
    $\initial4$    & $\initial4$ &  $\initial4$ \\
  \end{tabular}
\end{center}
The net meets the required properties (Proposition~\ref{prop:example} in appendix),
but neither Theorem~\ref{thm:contractive} 
nor Theorem~\ref{thm:monotone}; 
illustrating that the class extends beyond the two sufficient conditions.
\Omit{
  The one-step IFP formulas are:
  \begin{alignat*}{1}
    \psi_1 ( \bm{R}; u ) &\equiv       \black(u) \land \forall v( \neg E(u,v) ) \,, \\
    \psi_2 ( \bm{R}; u ) &\equiv       \red(u) \land \forall v( \neg E(u,v) ) \,, \\
    \psi_3 ( \bm{R}; u ) &\equiv       \black(u) \land \exists v( E(u,v) \land \curr1( \bm{R}; u) ) \\
                         &\qquad\quad\ \land \forall v( E(u,v) \implies \curr1( \bm{R}; u ) \,, \\
    \psi_4 ( \bm{R}; u ) &\equiv  \bigl(\red(u) \land \exists v( E(u,v))) \lor \\
                         &\bigl(\black(u) \land \exists v \bigl( E(u,v) \land \textstyle \bigvee_{k=2,3,4} \curr{k}( \bm{R}; v \bigr) \bigr) \,, \\
    \curr{k}( \bm{R}; u) &\equiv       R_k(u) \land \textstyle\bigwedge_{\initial{k}\prec\initial{j}} \neg R_j(u) \,.
  \end{alignat*}
Unlike reachability, the network does not admit an obvious description.
Nevertheless, the compiler applies and outputs
an IFP formula, mechanically and without human insight.
}

\begin{toappendix}

  \medskip\noindent
  The next proposition shows that the contractive net of the second example satisfies
  the conditions for applying the compiler. Namely, it is a converging regular net
  with finite vocabulary and whose game is admissible, ordered and monotone.

  \begin{apxproposition}[Second example]
    \label{prop:example}
    The net of the second example  
    is converging and regular; its vocabulary is 
    $\V=\set{\initial1\,{=}\,\tfrac{1}{3}, \initial2\,{=}\,\tfrac{2}{3}, \initial3\,{=}\,\tfrac{7}{9}, \initial4\,{=}\,1}$
    with the table of the main text; and its game is admissible, ordered by $\leq$, and monotone.
  \end{apxproposition}
  \begin{proof}
    Write $m_\ell(u) = \max\set{z_{v,\ell} \mid E(x,u,v)}$ (with value 0 for vertices without successors)
    so that $z_{u,\ell+1} = \clip(-z_{u,\ell}/2 + 2m_\ell(u) + \bm{C}_{\atp(u)})$.

    \medskip\noindent\textbf{Drift Lemma.} If $z_{\ell+1} = \clip(-z_\ell/2 + c_\ell)$
    with $c_\ell \to c^*$, then $z_\ell \to z^*$, the unique fixed point of
    $z \mapsto \clip(-z/2 + c^*)$: since $\clip$ is 1-Lipschitz,
    $|z_{\ell+1} - z^*| \leq \tfrac{1}{2}|z_\ell - z^*| + |c_\ell - c^*| \to 0$.

    \medskip\noindent\textbf{Convergence and limits.}
    After one step every vertex satisfies $z \geq \tfrac{1}{2}$:
    for black, $z_1 = \clip(-0/2 + 2m_0 + 1/2) = \clip(2m_0 + 1/2) \geq 1/2$, and
    for red, $z_1 = \clip(-1/2 + 2m_0 + 1) = \clip(2m_0 + 1/2) \geq 1/2$.
    Consequently, at step 2 every vertex with at least one successor has
    $z_2 = \clip(-z_1/2 + 2m_1 + \bm{C}) \geq \clip(-1/2 + 1 + 1/2) = 1$, i.e., exactly 1.
    Now classify vertices by $h(u) = \text{length of the longest path from $u$}$:
    \begin{enumerate}[leftmargin=*, label=(\alph*)]
      \item $h(u)=0$ (no successors): $z_{\ell+1} = \clip(-z_\ell/2 + \bm{C}_\alpha)$,
        a contraction with fixed point 1/3 (black) or 2/3 (red); the run converges there, oscillating.
      \item $h(u) = \infty$ ($u$ reaches a cycle): every cycle vertex has a successor, so it equals
         1 at step 2, and its aggregate from step 2 on includes its cycle successor, keeping
         it at 1 thereafter; by induction on the distance to the cycle, every vertex
         reaching a cycle is 1 from some step on, so $\Phi_u = 1$.
       \item $0 < h(u) < \infty$: by induction on $h(u)$, each successor's run converges, so
         $m_\ell(u) \to m^* = \text{max of the successors' limits}$, and the drift lemma
         gives $\Phi_u$ equals the unique fixed point of $z \mapsto \clip(-z/2 + 2m^* + \bm{C}_\alpha)$.
         Evaluating the fixed points:
           $m^* = 0$ gives 1/3 (black) and 2/3 (red);
           $m^* = 1/3$ gives 7/9 (black: $-z/2 + 7/6$ has fixed point 7/9) and 1 (red: $\clip(-1/2 + 5/3) = 1$); and
           $m^* \geq 2/3$ gives 1 for both types.
    \end{enumerate}
    This reproduces the table row by row, shows every reachable limit lies in \set{1/3, 2/3, 7/9, 1},
    and closing $\V_0 = \set{1/3, 2/3}$ under the rows yields exactly $\V$.
         
    \medskip\noindent\textbf{Regularity.}
    By (a)–(c), $\Phi_u$ equals the fixed point of $z \mapsto \clip(-z/2 + 2m^* + \bm{C}_\alpha)$
    with $m^*$ equal to the max of the successors' limits; since the map is a contraction in $z$,
    the frozen run from $\initial\alpha$ converges to that same fixed point, i.e., $\Phi_u = \Eop\alpha(S^*_u)$.

    \medskip\noindent\textbf{Game.}
    $\GEmbu{0}{u} = \T(\alpha,\emptyset) \in \set{\initial1, \initial2}$.
    Vertices without successors stay constant.
    A vertex whose successors all lack successors reaches its final row value at $k = 1$
    and stays (its successors are constant): runs $\initial1 \preceq \initial2 \preceq \initial3 \cdots$ or jumps to $\initial4$.
    Any other vertex has a successor whose value is $\geq\initial3$ from $k=1$ on,
    hence equals $\T(\alpha, \max \geq 7/9) = \initial4$ from $k=2$ on.
    All runs are $\preceq$-chains (the table rows are non-decreasing in max $S$),
    the limits coincide with $\Phi$ case by case, and  monotonicity holds since
    $S \sqsubseteq S'$ implies $\max S \leq \max S'$ and the rows are non-decreasing
    in $\max S$ (this last because if every $\bm{e} \in S$ is dominated by some
    $\bm{e}' \in S'$, then $\max S \leq \max S'$ since the vocabulary is non-negative
    and totally ordered, and the rows are non-decreasing in $\max S$ by inspection
    of the table; cf.\ Corollary~\ref{corollary}).
  \end{proof}
\end{toappendix}

\newcommand{\muf}{\ensuremath{\gamma}}
\newcommand{\IDX}{\mathrm{idx}}
\newcommand{\idx}[1]{\IDX_{#1}}
\newcommand{\Init}{\mathrm{Init}}
\newcommand{\FP}{\mathrm{FP}}
\newcommand{\lessp}[2]{#1 \sqsubseteq #2}
\newcommand{\sem}{\mathrm{sem}}
\newcommand{\denot}[2]{\ensuremath{(\!( #1 )\!)_{#2}}\xspace}

\section{Equivalence with \bsd}
\nosectionappendix
\begin{toappendix}
  \myappendixheader{Appendix for section: Equivalence with \bsd}
\end{toappendix}

We now establish the paper's central result: the vertex
queries implementable by \ratchet nets are exactly those
definable in the fragment \bsd of modal $\mu$-calculus.
\Omit{
  denoted by \Lmu, a propositional modal logic.\footnote{Both formalisms define
    unary queries: \set{u\in V(x) \mid x \vDash \Phi(u)} for
    IFP formula $\Phi(z)$, and \set{u\in V(x) \mid (x,u) \vDash \Psi}
    for \Lmu formula $\Psi$.
    Equivalence means these subsets coincide on every finite graph;
    the explicit free variable is replaced by the pointed-structure
    convention, a change of syntax, not of what is computed.
  }
}
The formulas in modal $\mu$-calculus, denoted as \Lmu, over a set
of atomic propositions $P$ and propositional variables $X$ are:
\[ \varphi ::= p \mid \neg p \mid \varphi\land\varphi \mid \varphi\lor\varphi \mid \mydiamond\varphi \mid \mybox\varphi \mid X \mid \mu X.\varphi \mid \nu X.\varphi \]
For the evaluation of formulas $\varphi$ at pointed graphs $(x,u)$ (semantics)
see elsewhere, e.g., \cite{arnold2001rudiments}.

A closed formula is in \sd if it is built from literals ($p, \neg p$),
conjunction, disjunction, $\mydiamond$, and $\mu$ only; it is in \pb if it is
built from literals, conjunction, disjunction, $\mybox$, and $\nu$ only.
The two classes are dual: negating a \sd formula and pushing negation to the atoms
yield a \pb formula and vice versa.
\bsd is the closure of $\sd\cup\pb$ under Boolean connectives; equivalently,
the Boolean closure of \sd alone.
Every \bsd formula is alternation-free, but \bsd does not contain the first level
$\Sigma^\mu_1 \cup \Pi^\mu_1$: it excludes least fixed points that use $\mybox$ and
greatest fixed points that use $\mydiamond$, such as $\mu X.\mybox X$ (well-foundedness)
and $\nu Y.\mydiamond Y$ (existence of an infinite path).
Proposition~\ref{prop:barrier} shows this exclusion is forced: those
formulas are not computable by any \ratchet net, so \bsd, not the first level
and not the alternation-free fragment, is the exact target for equivalence.

\bsd has a semantic reading that matches our dynamics. A \sd formula true at
$u$ is true on some finite truncation of the unfolding and remains
true on every extension. 
Dually, a \pb formula that fails is refuted on a finite truncation.
\bsd is the Boolean closure of finitely verifiable and finitely refutable
properties, the branching-time analogue of the obligation class of
\citet{MannaPnueli1990}.

\begin{theoremrep}[Net-to-logic]
  \label{thm:net-to-logic}
  Let $\Net(\cdot,\cdot,\cdot)$ be a \ratchet net with vocabulary
  $\V=\set{\initial{1},\ldots,\initial{m}}$ of size $m$.
  There are \bsd formulas $\muf_1,\ldots,\muf_m$
  such that for any pointed graph $(x,u)$ and $i\,{\in}\,\set{1,\ldots,m}$:
  $\Phi(\itype(x,u))=\initial{i}$ iff $(x,u)\vDash\muf_i$.
\end{theoremrep}
\begin{proof}
  Let $x$ be a finite graph. Define the subsets
  $T_i = \set{ u \in V(x) \mid \Phi(\itype(x,u)) \succeq \initial{i} }$ for $i=1,\ldots,m$.
  For subsets $X_1,\ldots,X_m\subseteq V(x)$, define $\Theta_i(v; X_1, \ldots, X_m)$ as
  \[ \bigvee_{(\alpha,S):\T(\alpha,S)\succeq\initial{i}} \biggl( \alpha(v) \land \bigwedge_{j\in S} \exists w \bigl( E(v,w) \land X_j(w) \bigr) \biggr) \,. \]
  These are guarded formulas jointly positive in $X_i$ by construction (no negation appears).
  Note also that $\Theta_i$ contains no universal quantification over neighbors,
  unlike Definition~\ref{def:type-formula}, only the witnessing half of the
  neighborhood description is asserted.
  This is sound because $T_i$ is an up-set and the game is monotone (additional
  neighbor values can only move $\T(\alpha,\cdot)$ up the order; this is the content
  of the ($\Leftarrow$) direction of Lemma~\ref{lemma:soundness}), and it is what
  makes the compiled system $\mydiamond$-only.

  The formula $\Theta_i$ defines a vertex subset $F_i(X_1,\ldots,X_m) = \set{ v \in V(x) \mid x \vDash \Theta_i(v; X_1, \ldots, X_m) }$.
  Putting all of them together, we get the operator $\bar F: \powerset(V(x))^m \to \powerset(V(x))^m$
  that maps $(X_1,\ldots,X_m) \mapsto \bar F(X_1,\ldots,X_m) = (F_i(X_1,\ldots,X_m))_{i}$.
  This operator is monotone and it has a least fixed point ($\subseteq$-component wise)
  given by $(T_1,\ldots,T_m)$. Indeed,

  \begin{innerlemmarep}[Soundness]
    \label{lemma:soundness}
    For every $i=1,\ldots,m$ and vertex $v$:
    $v\in T_i \iff x \vDash \Theta_i(v; T_1, \ldots, T_m)$.
  \end{innerlemmarep}
  \noindent\emph{Proof of Lemma:}
    Let $S^*(v) = \set{ \initial{w} \mid E(v,w) }$ where $\initial{w}$ denotes $\Phi(\itype(x,w))$.
    By regularity and admissibility, $\Phi(\itype(x,v))=\T(\alpha, S^*(v))$ where $\alpha=\atp(v)$.
    \begin{enumerate}[leftmargin=*, label=$\bullet$]
      \item ($\Rightarrow$) If $\Phi(\itype(x,v)) \succeq \initial{i}$, the row
        $(\alpha,S)$ for $S=S^*(v)$ satisfies $\T(\alpha, S) \succeq \initial{i}$,
        and every $\initial{j}\in S$ is witnessed by an actual neighbor $w$
        (i.e., $w\in T_j$). So, $x \vDash \Theta_i(v; T_1, \ldots, T_m)$.
      \item ($\Leftarrow$) Suppose $(\alpha,S)$ is a row such that $T(\alpha,S)\succeq\initial{i}$
        and for each $\initial{j}\in S$, there is a neighbor $w$ such that
        $\initial{w}\succeq\initial{j}$. Then, $S \sqsubseteq S^*(v)$.
        By monotonicity of the game, $\T(\alpha, S) \preceq \T(\alpha, S^*(v))$.
        That is, $\initial{i} \preceq \T(\alpha, S) \preceq \T(\alpha, S^*(v)) = \Phi(\itype(x,v))$
        which implies $v\in T_i$.
        This is the only place in the proof where monotonicity of the game is needed.
        \hfill\ensuremath{\scaleobj{.75}{\blacksquare}} 
    \end{enumerate}
 
  \begin{innerlemmarep}[Minimality]
    \label{lemma:minimality}
    If $\bar T' = (T'_1,\ldots,T'_m)$ is any tuple of vertex subsets such that
    $x\vDash\Theta_i(v;\bar T') \implies v\in T'_i$ for every $i$, then
    $T_i \subseteq T'_i$ for every $i$.
  \end{innerlemmarep}
  \noindent\emph{Proof of Lemma:}
    We show by induction on $k$ the claim: $\GEmbu{k}{u} \succeq \initial{i} \implies u \in T'_i$
    for every $k$, $i$ and $u$.
    \begin{enumerate}[leftmargin=*, label=$\bullet$]
      \item \emph{Base case.} $\GEmbu{0}{u} = \T(\alpha,\emptyset)$.
        If $\T(\alpha,\emptyset) \succeq \initial{i}$, the row $(\alpha,\emptyset)$
        satisfies $\T(\alpha,\emptyset)\succeq\initial{i}$ giving $x \vDash \Theta_i(u; \bar T')$.
        Hence, by assumption, $u \in T'_i$.
      \item \emph{Inductive step.}
        Suppose the claim holds at $k$, and let $\GEmbu{k+1}{u}=\T(\alpha,S^k_u)$
        where $S^k_u=\set{\GEmbu{k}{v} \mid E(u,v)}$.
        If $\GEmbu{k+1}{u}\succeq\initial{i}$, the row $(\alpha,S^{k}_u)$ satisfies
        $\T(\alpha,S^k_u)\succeq\initial{i}$.
        For each $\initial{j}\in S^k_u$, the witnessing neighbor $v$ has
        $\GEmbu{k}{v}=\initial{j} \succeq \initial{j}$.
        By the inductive hypothesis, $v\in T'_j$.
        Then, $x \vDash \Theta_i(u; \bar T')$ via this row, hence $u \in T'_i$.
    \end{enumerate}
    Since $\V$ is finite and $\GEmb{k}{u}\in\V$ for all $k$,
    $\GEmbu{k}{u} \to \Phi(\itype(x,u))$ forces eventual equality:
    there is $K$ such that $\GEmbu{k}{u}=\Phi(\itype(x,u))$ for $k\geq K$
    (a convergent sequence in \Embeddings taking values in a finite set is eventually constant).
    Let $u$ be a vertex in $T_i$ (i.e., $\Phi(\itype(x,u))\succeq\initial{i}$).
    Therefore $\GEmbu{K}{u}\succeq\initial{i}$, and by the claim $u \in T'_i$.
    Hence, $T_i\subseteq T'_i$. \hfill\ensuremath{\scaleobj{.75}{\blacksquare}}

  \bigskip\noindent
  \textbf{Construction of the formulas $\muf_i$}

  \medskip\noindent
  First, translate the first-order formulas $\Theta_i$ into \Lmu formulas $\zeta_i$,
  syntactically, using
  \[ \alpha(v) \mapsto \alpha\,; \quad \exists w(E(v,w)\land X_j(w)) \mapsto \mydiamond X_j \]
  while preserving conjunctions and disjunctions.

  \begin{innerlemmarep}
    \label{lemma:translation}
    For every subsets $A_1,\ldots,A_m \subseteq V(x)$ and $i=1,\ldots,m$:
    $x \vDash \Theta_i(v; A_1, \ldots, A_m)) \iff
       (x,v) \vDash \zeta_i[ X_1, \ldots, X_m := A_1, \ldots, A_m]$.
  \end{innerlemmarep}
  \noindent\emph{Proof of Lemma:}
    Direct. \hfill\ensuremath{\scaleobj{.75}{\blacksquare}}

  \bigskip\noindent
  Apply now variable elimination (or \emph{Beki\v{c}'s principle}) \cite{arnold2001rudiments}
  to get formulas $\muf_i$ that characterize the least fixed-point $(T_1,\ldots,T_m)$ of $\bar F$.
  That is, define $\eta_1(X_2,\ldots,X_m) = \mu X_1.\zeta_1(X_1,\ldots,X_m)$ and define
  for $j=2,\ldots,m$, $\zeta^{(1)}_j = \zeta_j[X_1 := \eta_1(X_2,\ldots,X_m)]$.
  Next, define $\eta_2(X_3,\ldots,X_m) = \mu X_2.\zeta^{(1)}_2(X_2,\ldots,X_m)$
  and $\zeta^{(2)}_j = \zeta^{(1)}[X_2 := \eta_2(X_3,\ldots,X_m)]$ for $j=3,\ldots,m$.
  Continue until $\eta_m = \mu X_m.\zeta^{(m-1)}_m(X_m)$ has no variables.

  Construct $\theta_m\,{=}\,\eta_m$, and for $j=m-1,\ldots,1$:
  \[ \theta_j = \eta_j[X_m, \ldots, X_{j+1} := \theta_m, \ldots, \theta_{j+1}] \,. \]
  These are \Lmu formulas whose only operator is $\mu$.
  Finally, set $\theta_{m+1}=\bot$ and
  $\muf_i = \theta_i \land \bigwedge_{\text{$\initial{j}$ covers $\initial{i}$}} \neg \theta_j$
  where the cover of $\initial{i}$ is the set of minimal embeddings that are strictly above $\initial{i}$,
  and $\neg\theta_j$ is obtained syntactically by pushing the negation down to
  propositions, resulting in a formula with only $\nu$ operators.
  $\muf_i \in \bsd$: each $\theta_i$ is a closed \sd formula, and each
  $\neg\theta_i$, obtained by dualization, is a closed \pb formula.

  \bigskip\noindent
  \textbf{Correctness}

  \medskip\noindent
  We need to show $\Phi(\itype(x,u))=\initial{i} \iff (x,u) \vDash \muf_i$.
  Let $\bar T = (T_1, \ldots, T_m)$. By Lemma~\ref{lemma:soundness}, $\bar T$
  is a fixed point of operator $\bar F$, and Lemma~\ref{lemma:minimality}
  shows $\bar T$ is the least fixed point.

  By Lemma~\ref{lemma:translation}, the operator induced by $\zeta_1,\ldots,\zeta_m$
  under ordinary modal semantics is the same operator $\bar F$.
  Thus, the least-fixed point of the $\zeta$-system is also $\bar T$.
  By Beki\v{c} elimination, the constructed $\theta_1,\dots,\theta_m$ satisfy
  $x \vDash \theta_i(u)$ iff $u \in (\mathrm{lfp}(\bar F))_i = T_i$, for every $i$.

  Finally, $\Phi(\itype(x,u))=\initial{i}$ iff $u\in T_i$ and $u \notin T_j$ for
  $\initial{j}\succ\initial{i}$ iff $(x,u)\vDash\theta_i$ and $(x,u)\nvDash\theta_j$
  iff $(x,u) \vDash \muf_i$.
\end{proof}

Theorem~\ref{thm:net-to-logic} answers, for the counting-free, \bsd fragment,
the converse question that \citet{bollen2025:halting} leave open in the graded setting.
Its proof compiles the table into an intermediate system
of mutually recursive, guarded, negation-free fixed-point definitions (one per vocabulary
element, with one disjunct per row of $\T$) and then flattens this system into closed
\Lmu formulas with Beki\v{c}'s elimination \cite{arnold2001rudiments}.
This two-stage process determines the size of the output, which we measure against the
quotient table $\T$. 
The size of the intermediate system, equivalently, the simultaneous IFP formula in
Theorem~\ref{thm:rec-compiler}, is linear in $|\T|$ when measured as a system with
shared definitions (each row and each $\curr{i}$ defined once and referenced);
written out as a tree it incurs an additional factor $|V|$.
The closed formulas in \Lmu may be of exponential size because of
Beki\v{c}'s elimination. 
Theorem~\ref{thm:net-to-logic} provides two guarantees of different strength:
the closed terms $\gamma_i$ always exist and are computable from the net, 
while the linear size bound holds for the system of definitions from which they are flattened.
Considering the system as the primary output of the Theorem is consistent with practice
in verification, where the linear-time model-checking algorithms for the alternation-free
$\mu$-calculus evaluate such equation systems directly, without constructing
a flattened closed formula \cite{andersen1994,cleaveland1991linear,vergauwen1994efficient}.

The other direction maps \bsd formulas over propositions that capture atomic
types into \ratchet nets. Propositions that do not correspond to atomic types
can be accounted for by considering the different Boolean valuations.

\begin{theoremrep}[Logic-to-net]
  \label{thm:logic-to-net}
  For every \bsd formula $\gamma$,
  there is a \ratchet architecture $\Net(\cdot,\cdot,\cdot)$ and register
  index $m$ such that for any pointed graph $(x,u)$: $(x,u)\vDash\muf$
  iff $[\Phi(\itype(x,u))]_m=1$.
\end{theoremrep}
\begin{proof}
  By definition, we assume $\gamma$ is equivalent to
  $\mathrm{\bf B}(\beta_1,\ldots,\beta_q,\neg\beta_{q+1},\ldots,\neg\beta_r)$,
  where $\mathrm{\bf B}$ is a Boolean combination $(\land,\lor)$ and
  each $\beta_j$ is a closed \sd formula.
  Further, we assume the bound variables are uniquely named.

  \medskip\noindent\textbf{Registers.}
  Enumerate the subformulas of the blocks bottom-up as $\gamma_1,\ldots,\gamma_m$
  (repetitions preserved) and let $b_i$ index the bodies such that
  $\gamma_m=\gamma$ and $\gamma_i = \mu X_i.\gamma_{b_i}$ when $\gamma_i$ is a fixed-point formula.
  Since blocks are \sd, every bound variable is a $\mu$-variable.
  The width of the net is $W=m+M$, with one register $r_i$ per subformula
  and one register $v_i$ per bound variable.
  Registers are partitioned into \textbf{core} registers (those for subformulas
  $\gamma_i$ lying inside a block and those for bound variables),
  and \textbf{derived} registers (the rest).
  Registers $v_i$ are also called ``latch registers'', and the others ``memoryless''.

  \medskip\noindent\textbf{Clusters.}
  For each block $\beta$, consider the directed graph over the bound variables
  in $\beta$, with edges $X\to Y$ if $Y$ appears within the scope of $X$, and
  edges $Y\to X$ if $X$ appears as a free variable in the body for $Y$.
  A cluster $C$ is a strongly connected component in this graph, and the
  variable $X$ is the head of $C$ iff every other variable $Y$ in $C$ is
  within the scope of $X$.

  \medskip\noindent\textbf{Atomic types.}
  As discussed above, we take the propositions in $\gamma$ to be the atomic types
  of the graphs; $\alpha \vDash p$ abbreviates that type $\alpha$ makes $p$ true.
  When the propositions are not disjoint or exhaustive, we shall consider their
  jointly Boolean valuations as the set of atomic types.

  \medskip\noindent\textbf{Initial embeddings and rules.}
  The initial embedding $\initial\alpha$ assigns $\bracket{ \alpha \vDash p }$ to
  registers of literal subformulas $p$ and $\neg p$ (their value under $\alpha$),
  and 0 to every other register.
  The update rules implemented by the parameters, ordered by register type, are:
  \begin{center}
    \begin{tabular}{cl}
      \toprule
      Form of $\gamma_k$                  & \multicolumn{1}{l}{Posted rule for register $r_k$} \\
      \midrule
      $p$                                 & $r_k := \bracket{ \atp(u) \vDash p }$ \\[3pt]
      $\neg p$                            & $r_k := 1 - \bracket{ \atp(u) \vDash p }$ \\[3pt]
      $\gamma_i \land \gamma_j$           & \makecell[l]{$r_k := \clip(r_i + r_j - 1)= \min(r_i, r_j)$} \\[3pt]
      $\gamma_i \lor  \gamma_j$           & \makecell[l]{$r_k := \clip(r_i + r_j) = \max(r_i, r_j)$} \\[3pt]
      $\mydiamond\gamma_i$                & $r_k := \bigoplus_S^{\max} r_i[v] = \max(\set{ r_i[v] \mid v \in S })$ \\[3pt]
      $\neg\gamma_i$                      & $r_k := 1 - r_i$ \\[3pt]
      $X_i$                               & $r_k := v_i$ \\[3pt]
      \makecell[c]{$\gamma_k = \mu X_k.\gamma_{b_k}$ \\ ($X_k$ is head of cluster $C$)} &
      $\begin{cases}
        r_j := v_j                                        & \qquad \text{(for $X_j \in C$)} \\
        v_j := \max(v_j, r_{b_j}) = \clip(v_j + r_{b_j})  & \qquad \text{(for $X_j \in C$)}
      \end{cases}$ \\[3pt]
      \makecell[c]{$\gamma_k = \mu X_k.\gamma_{b_k}$ \\ ($X_k$ isn't head)} & (Nothing; handled when solving the cluster's head) \\
      \bottomrule
    \end{tabular}
  \end{center}
 
  All rules fire synchronously, reading the previous step. The only aggregation
  is max; under the global convention $\oplus\emptyset = 0$ this evaluates $\mydiamond$
  correctly at vertices without successors ($\mydiamond\varphi$ is false there).
  All rules are compositions of clip with affine maps of the current embedding
  and the aggregate, so the construction is an instance of the architecture
  with $\oplus$ the pointwise maximum.
  All registers take values in $\set{0,1}$ (each rule maps Booleans to Booleans),
  so the reachable embeddings lie in $\set{0,1}^W$.

  \medskip\noindent\textbf{Semantic states.}
  Fix a graph $x$ and let $\rho_x$ be the environment assigning to each bound variable
  $X_i$ of $\gamma$ its denotation $\denot{X_i}{x} \subseteq V(x)$ in the simultaneous
  least solution of its block's equation system (equivalently, by Beki\v{c}'s principle,
  the nested least-fixed-point semantics).
  The \emph{semantic state} $z^\sem_u \in \set{0,1}^W$ at vertex $u$ is:
  \begin{enumerate}[leftmargin=*, label=--]
    \item $z^\sem_u[r_i] = \bracket{ (x,u,\rho_x) \vDash \gamma_i }$ for every subformula register, and
    \item $z^\sem_u[v_i] = \bracket{ u \in \denot{X_i}{x} }$ for every variable register.
  \end{enumerate}
  By definition, $(x,u,\rho_x) \vDash \gamma_i$ iff $z^\sem_u[r_i] = 1$. In particular,
  if $\gamma_i=\beta$ is a closed block, $(x,u) \vDash \beta$ iff $z^\sem_u[r_i] = 1$.

  \medskip

  \begin{innerlemma}[Iteration]
    \label{lemma:kleene}
    Let $g$ be a monotone self-map of $\set{0,1}^D$ for a finite set $D$, and let $z_0$
    satisfy $z_0 \leq g(z_0)$ and $z_0 \leq z^*$ for every fixed point $z^*$ of $g$.
    Then the iteration $z_{\ell+1} = g(z_\ell)$ is non-decreasing, stabilizes within
    $|D|$ steps, and its limit is the \emph{least fixed point} of $g$.
  \end{innerlemma}
  \noindent\emph{Proof sketch:}
    By induction: $z_0 \leq z_1 \leq \cdots \leq z^*$ for every fixed point $z^*$ of $g$.
    Each coordinate is Boolean and flips at most once.
    Hence, the run converges within $|D|$ steps to the least fixed point of $g$.
    \hfill\ensuremath{\scaleobj{.75}{\blacksquare}}    


  \begin{innerlemma}[Fixed points of the block map]
    \label{lemma:fp}
    Fix a graph $x$ and a block $\beta$ with register set $R_\beta$, and let $f$
    be the synchronous one-step map of the net restricted to $R_\beta$, a self-map
    of $\set{0,1}^{R_\beta \times V(x)}$. Then,
    \begin{enumerate}[leftmargin=*, label=$(\alph*)$]
      \item $f$ is monotone;
      \item a state $z$ is a fixed point of $f$ iff every memoryless register satisfies
        its rule with equality and every latch satisfies $v_j \geq r_{b_j}$,
      \item the variable part of every fixed point of $f$ is a pre-fixed point of the body
        operator $P$, where $P(w)_{j,u}$ is the value of $\gamma_{b_j}$ at $u$ under variable
        assignment $w$, while $z^\sem$ is a fixed point of $f$ whose variable part is $\mathrm{lfp}(P)$.
    \end{enumerate}
    Consequently, $\mathrm{lfp}(f) = z^\sem$ on $R_\beta$.
  \end{innerlemma}
  \noindent\emph{Proof:}
    \begin{enumerate}[leftmargin=*, label=$(\alph*)$]
      \item literal registers are per-type constants; every other core rule is a composition of
        min, max, copies, and max-aggregates of neighbor registers, each monotone. 
      \item For memoryless rules, fixedness is the equality; for a latch, $v_j = \max(v_j, r_{b_j})$
        iff $v_j \geq r_{b_j}$.
      \item Forward: composing the exact memoryless equations of a fixed point $z^*$ from each body
        down to variables and constants, the latch inequalities read $w \geq P(w)$ on the variable part.
        Backward: the rules mirror the semantic clauses (min/max implement $\land\,/\,\lor$; the max-aggregate
        implements $\mydiamond$, with $\max\emptyset = 0$ matching falsity at successor-less vertices;
        copies match by definition of $\rho_x$), and at each latch $z^\sem[v_j]$ equals the body's value,
        so the inequality holds with equality; the variable part of $z^\sem$ is $\mathrm{lfp}(P)$ by Beki\v{c}'s
        principle.
    \end{enumerate}
    For the consequence: $\mathrm{lfp}(f) \leq z^\sem$ since $z^\sem$ is a fixed point.
    Conversely any fixed point $z^*$ has variable part $\geq \mathrm{lfp}(P)$ by Knaster–Tarski via (c),
    and its memoryless registers, exact monotone functions of the variables and constants, dominate as well;
    hence $z^* \geq z^\sem$.
    \hfill\ensuremath{\scaleobj{.75}{\blacksquare}}

  \medskip
  Lemma~\ref{lemma:fp} applies verbatim when the aggregates of some or all $\mydiamond$-registers are frozen
  at constants from a finite set $S \subseteq \set{0,1}^W$: freezing only changes the constants of the memoryless
  equations, not the shape of the system.

  \begin{innerlemma}[Blocks reach their semantics]
    \label{lemma:blocks}
    Let $\beta$ be a block of $\gamma$ with register set $R_\beta$, and fix a graph $x$
    with $n$ vertices. Restricted to $R_\beta$, the synchronous run of the net is non-decreasing,
    stabilizes within $|R_\beta|\times n$, and its limit is the semantic state:
    $\lim_\ell z_{u,\ell}[r] = z^\sem_u[r]$ for every register $r \in R_\beta$ and vertex $u$.
  \end{innerlemma}
  \noindent\emph{Proof:}
    The rules of $R_\beta$ read only registers of $R_\beta$ and per-type constants ($\beta$ is closed
    and its machinery internal to the block), so the restricted run is the iteration of the monotone
    map $f$ of Lemma~\ref{lemma:fp}(a).
    The initial state satisfies the hypotheses of Lemma~\ref{lemma:kleene}: literal registers start
    at their rules' constant values, which every fixed point of $f$ shares by Lemma~\ref{lemma:fp}(b),
    and all other registers start at 0, so $z_0 \leq f(z_0)$ and $z_0 \leq z^*$ for every fixed
    point $z_*$. By Lemma~\ref{lemma:kleene} the run is non-decreasing and stabilizes within
    $|R_\beta|\times n$ steps at $\mathrm{lfp}(f)$, which is $z^\sem$ by Lemma~\ref{lemma:fp}.
    \hfill\ensuremath{\scaleobj{.75}{\blacksquare}}

  \begin{innerlemma}[Derived registers]
    \label{lemma:derived}
    If the inputs of a derived register converge, the register converges to the
    corresponding connective of the limits, with delay one.
  \end{innerlemma}
  \noindent\emph{Proof sketch:}
    Each derived rule is a memoryless continuous function of its inputs;
    the embeddings always lie in $\set{0,1}^W$.
    \hfill\ensuremath{\scaleobj{.75}{\blacksquare}}

  \medskip\noindent\textbf{Step 1: run of the net.}
  Fix a graph $x$ with $n$ vertices.
  By Lemma~\ref{lemma:blocks} each block converges to its semantics.
  By Lemma~\ref{lemma:derived}, $r_m \to \bracket{ (x,u) \vDash \gamma }$ at every vertex $u$.
  The net converges on every graph within $W\times n$ steps.

  \medskip\noindent\textbf{Step 2: vocabulary and regularity.}
  For any finite $S \subseteq \set{0,1}^W$, the frozen run
  $z_0 = \initial\alpha, z_{\ell+1} = \Net(\alpha, z_\ell, S)$
  is the single-vertex instance of Step 1 with the aggregates
  held constant.
  By the remark after Lemma~\ref{lemma:fp}, this frozen system
  is again the iteration of a monotone map from a qualifying start,
  so by Lemma~\ref{lemma:blocks} it converges; derived registers
  follow by Lemma~\ref{lemma:derived}.
  Therefore $\Eop\alpha(S)$ is defined, the vocabulary $\V$ is
  well defined, and $\V \subseteq \set{0,1}^W$ is finite.

  By Lemma~\ref{lemma:fp} applied to the frozen instance, $\Eop\alpha(S)$ is,
  on each block's core, the least solution of the memoryless equations with
  aggregates frozen at $S$, subject to the latch inequalities.
  For regularity, fix $(x, u)$ with $\alpha = \atp(u)$ and
  $S_u = \set{\Phi_v \mid E(u,v)}$; blocks are closed, hence independent,
  and derived registers are determined by the core, so it suffices to
  prove $\Phi_u = \Eop\alpha(S_u)$ on each block's core.

  $(\leq)$ At the net's limit all of the block's equations hold at $u$
  with the neighbors at their limits, so $\Phi_u$ restricted to the core
  solves the frozen system with aggregates from $S_u$; by leastness,
  $\Eop\alpha(S_u) \leq \Phi_u$.

  $(\geq)$ Consider the global assignment that places $\Eop\alpha(S_u)$'s
  core values at $u$ and $\Phi$'s core values at every other vertex.
  At $u$ all constraints hold by construction, since $u$'s aggregates
  under this assignment are exactly those frozen from $S_u$.
  At any other vertex $w$, $u$'s values enter only through max-aggregates
  and were only lowered, so by monotonicity of the rules $w$'s equalities
  relax to inequalities of the form $f(z) \leq z$.
  The assignment is therefore a global pre-fixed point of the block map $f$,
  hence dominates $\mathrm{lfp}(f)$ by Knaster–Tarski, and $\mathrm{lfp}(f) = \Phi$
  on the block's core by Lemma~\ref{lemma:blocks}.
  Reading off vertex $u$ gives $\Eop\alpha(S_u) \geq \Phi_u$.
  Hence $\Phi_u = \Eop\alpha(S_u)$: the net is regular.

  \medskip\noindent\textbf{Step 3: the game is ordered, monotone, and admissible.}
  Define $\preceq$ on $\V$ by $\bm{e} \preceq \bm{e}'$ iff
  $\mathrm{core}(\bm{e}) \leq  \mathrm{core}(\bm{e}')$ componentwise.
  This is a partial order on $\V$: every element of $\V$ is a frozen-run limit,
  at which the derived equations hold exactly and no derived rule aggregates neighbors,
  so derived values are functions of same-vertex core values, and agreement on the core
  implies equality.
  Write $\sqsubseteq$ for the definition in Definition~\ref{def:game}.

  \medskip\noindent\emph{Monotone.}
  If $S \sqsubseteq S'$, then for every core register $r$ the aggregate satisfies
  $\max_{\bm{e}\in S} \bm{e}[r] \leq \max_{\bm{e}'\in S'} \bm{e}'[r]$:
  each element of $S$ is dominated by an element of $S'$, and extra elements
  can only increase a maximum. The least local solution is monotone in these
  frozen aggregates (standard monotonicity of least fixed points in parameters),
  so $\mathrm{core}(\T(\alpha,S)) \leq \mathrm{core}(\T(\alpha,S'))$,
  i.e., $\T(\alpha,S) \preceq \T(\alpha,S')$.

  \medskip\noindent\emph{Ordered.}
  $\GEmbu{0}{u} = \T(\alpha,\emptyset)$ and $\emptyset \sqsubseteq S$.
  Then, $\GEmbu{0}{u} \preceq \GEmbu{1}{u}$.
  If $\GEmbu{k}{v} \preceq \GEmbu{k+1}{v}$ for all $v$ then $S^k_u \sqsubseteq S^{k+1}_u$,
  and monotonicity gives $\GEmbu{k+1}{u} \preceq \GEmbu{k+2}{u}$.
  That is, every run is a $\preceq$-chain.

  \medskip\noindent\emph{Admissible.}
  $(i)$ $\GEmbu{k}{u} \preceq \Phi_u$ by induction:
  the base is $\T(\alpha,\emptyset) \preceq \T(\alpha,S_u) = \Phi_u$
  (monotonicity and regularity), and the inductive step is $\GEmbu{k+1}{u} = \T(\alpha,S^k_u) \preceq \T(\alpha,S_u) = \Phi_u$.
  $(ii)$ $z_{u,\ell} \preceq \GEmbu{\ell}{u}$ by induction:
  the base is $\mathrm{core}(\initial\alpha) \preceq \mathrm{core}(\T(\alpha,\emptyset))$
  (aggregates from $\emptyset$ are the minimum, and the frozen run ascends from $\initial\alpha$).
  For the inductive step, $\GEmbu{\ell}{u} \preceq \Eop\alpha(S^\ell_u)$
  (for $\ell=0$ by monotonicity from $\emptyset \sqsubseteq S^0$; for $\ell\geq1$ since $\GEmbu{\ell}{u} = \Eop\alpha(S^{\ell-1}_u) \preceq \Eop\alpha(S^\ell_u))$,
  and then $z_{u,\ell+1} = \Net(\alpha, z_{u,\ell}, \set{ z_{v,\ell} \mid E(u,v) } ) \leq \Net(\alpha, \Eop\alpha(S^\ell_u), S^\ell_u) = \Eop\alpha(S^\ell_u) = \GEmbu{\ell+1}{u}$
  on the core, using the inductive hypothesis, the monotonicity of the core rules, and the fixed-point property of
  $\Eop\alpha(S^\ell_u)$. By $(i)$ the core of $\tup{\GEmbu{k}{u}}_k$ is non-decreasing and bounded by $\Phi_u$,
  hence converges; by $(ii)$ its limit dominates $\Phi_u$; so it equals $\Phi_u$ on the core, and the
  derived registers, determined by the core on $\V$, follow. The game is admissible.

  \medskip\noindent The net is therefore a \ratchet net, and by Step 1, $[\Phi(\itype(x,u))]_m = \bracket{\gamma}$.
\end{proof}

\begin{toappendix}
  The construction in Theorem~\ref{thm:logic-to-net} cannot be pushed further.
  A $\mybox$ inside a $\mu$-block (or, dually, a $\mydiamond$ inside a $\nu$-block)
  would make the game's step-0 value $\T(\alpha,\emptyset)$ evaluate the modality
  vacuously -- correctly for the isolated vertex, but as an over-approximation that
  neighboring vertices then bootstrap from --
  and Proposition~\ref{prop:barrier} shows the resulting failure of admissibility is
  not an artifact of this construction: no admissible net computes $\mu X.\mybox X$ or
  $\nu Y.\mydiamond Y$. A fixed point over one of the opposite polarity fails one level
  higher for the same reason ($\text{EF EG } p$), which is where the halting machinery
  of \citeauthor{bollen2025:halting} earns its keep.
\end{toappendix}

Theorem~\ref{thm:logic-to-net} is the counting-free counterpart of \citeauthor{bollen2025:halting}'s
Theorem 5.1. Its proof realizes 
the classical cluster-elimination method for alternation-free equation systems.
The proof is aggregator-agnostic; only graded modalities would require threshold-sum
registers and hence multisets, and whether this route yields a simpler graded
construction than the halting-and-counting apparatus of \citeauthor{bollen2025:halting} is open.
What blocks going beyond \bsd is that stabilization over a finite vocabulary is a flat,
one-timescale process. 

\Omit{
  The proof is aggregator-agnostic; only graded modalities would require threshold-sum registers
  and hence multisets, and whether this route yields a simpler graded construction than the
  halting-and-counting apparatus of \citeauthor{bollen2025:halting} is open.
  What blocks going beyond \bsd is that stabilization over a finite vocabulary is a flat,
  one-timescale process: composing fixed points of opposite polarity requires a certificate
  that the inner computation has stabilized, a certificate a converging run can synthesize
  with counting \cite{bollen2026:halting-vs-converging} but provably cannot without it
  (Prop.~\ref{prop:barrier}).
}

Besides the theorems, regularity is identified as the semantic condition
matching $\mu$-calculus: both make the limit value at a vertex depend only
on the limit value at its neighbors, not on the trajectory by which the limits are
reached.
Admissibility and finiteness play complementary roles:
they are not needed to state what a regular net's fixed points are, but they make
those fixed points computable by a terminating procedure and the resulting formula
finite. 

The fragment is also a ceiling as shown by Proposition~\ref{prop:barrier}
which assumes only regularity, finite vocabulary, and admissibility but no ordered
or monotone game.
The proof exploits a blind spot of stabilization.
On a directed path and on a self-loop, every vertex feeds the game the same local information,
so both runs traverse the same sequence $t_0,t_1,\ldots$ of table values; since $\V$ is finite,
this sequence stabilizes at some $t_K$. Admissibility then assigns the head of any path of
length $k\geq K$, where no infinite path exists, and the self-loop, where one does, the same
limit $t_K$. Thus, any readout must give both instances the same verdict,
while $\nu Y.\mydiamond Y$ and its relatives assign them different truth values.
It is open whether the fragment is a ceiling for undirected graphs.

\begin{propositionrep}[Barrier]
  \label{prop:barrier}
  Let $\Net(\cdot,\cdot,\cdot)$ be a converging regular
  net with finite vocabulary whose game is admissible (ordering and
  monotonicity are not assumed).
  Then, $\Net$ does not compute the queries: $\nu Y.\mydiamond Y$,
  $\mu X.\mybox X$, $\text{EG } p = \nu Y.(p \land \mydiamond Y)$,
  and $\text{EF EG } p = \mu X.(\mydiamond X \lor \nu Y.(p \land \mydiamond Y))$.
\end{propositionrep}
\begin{proof}
  Formally, ``Net computes $q$'' means there is a readout $h : \V \to \set{0,1}$ with
  $h(\Phi(\itype(x,u))) = q(x,u)$ for every pointed graph; any thresholded register
  is such an $h$.
  Fix the atomic type $\alpha$ of interest (for the formulas $\mu X.\mybox X$ and
  $\nu Y.\mydiamond Y$ any $\alpha$;
  for the two $p$-formulas, an $\alpha$ with $\alpha\vDash p$) and define the sequence
  $t_0 = \T(\alpha,\emptyset)$ and $t_{k+1} = \T(\alpha, \set{t_k})$ in $\V$.
  Both families below use only vertices of type $\alpha$.

  \medskip\noindent
  \textbf{First family: directed paths}
  Let $P_k$ have vertices $u_k \to u_{k-1} \to \cdots \to u_0$, all of
  type $\alpha$, where $u_0$ is a sink.
  We claim $\GEmbu{j}{u_i} = t_{\min(i,j)}$ for all $i,j$ by induction on $j$.
  The base case is $\GEmbu{0}{u_i} = \T(\alpha,\emptyset) = t_0 = t_{\min{i,0}}$.
  For the inductive step: if $i=0$, then $u_i$ is a sink, so $\GEmbu{j+1}{u_0} = \T(\alpha,\emptyset)=t_0=t_{\min(0,j+1)}$.
  If $i\geq1$, then $u_i$'s only successor is $u_{i-1}$, so
  $\GEmbu{j+1}{u_i} = \T(\alpha,\set{\GEmbu{j}{u_{i-1}}}) = \T(\alpha, \set{t_{\min(i-1,j)}}) = t_{\min(i-1, j)+1} = t_{\min(i,j+1)}$. 
  Hence the game's run at $u_k$ constant from step $k$ on with value $t_k$,
  and admissibility gives $\Phi(\itype(P_k,u_k)) = t_k$, so $h(t_k)=q(P_k,u_k)$
  for every $k\geq 0$.

  \medskip\noindent\textbf{Second family: self-loop.}
  Let $L$ be the single vertex $w$ of type $\alpha$ with the edge $(w,w)$.
  Its game run is $\GEmbu{0}{w} = t_0$ and $\GEmbu{j+1}{w} = \T(\alpha,\set{\GEmbu{j}{w}})$,
  i.e. exactly $\tup{t_j}_{j\geq0}$.
  Since \V is finite and $t_{j+1}$ is a function of $t_j$, the sequence is eventually periodic.
  Since the game is admissible it converges in \Embeddings.
  An eventually periodic convergent sequence over a finite set is eventually constant.
  So there is a $K$ with $t_j=t_K$ for all $j\geq K$, and $\Phi(\itype(L,w))=t_K$,
  so $h(t_K)=q(L,w)$.

  \medskip\noindent\textbf{Contradiction.}
  Every path in $P_K$ is finite, and $w$ lies on an infinite path. Hence,
  \begin{enumerate}[leftmargin=*, label=$\bullet$]
    \item For $\nu Y.\mydiamond Y$, $\mathrm{EG } p$, and $\mathrm{EF EG }p$ (all of which
      assert the existence of an infinite path; in the last two cases, one path whose all vertices 
      satisfy $p$: $q(P_k, u_K) = 0$ and $q(L,w)=1$. Thus, $h(t_K) = 0$ and $h(t_K) = 1$.
    \item For $\mu X.\mybox X$ (well-foundedness) the two values are exchanged: the finite path is
      well-founded, so $q(P_K,u_K)=1$, while the self-loop is not, so $q(L,w)=0$.
      Again $h(t_K)=1$ and $h(t_K)=0$.
  \end{enumerate}
  In each case $h$ is forced to take both values at the single vocabulary element $t_K$,
  contradiction.
\end{proof}

\begin{corollaryrep}
  \label{cor:separation}
  Over finite pointed graphs, as classes of unary queries:
  $(i)$ $\Sigma^\mu_1 \cup \Pi^\mu_1 \not\subseteq \bsd$, witnessed by $\mu X.\mybox X \in \Sigma^\mu_1$ and $\nu Y.\mydiamond Y \in \Pi^\mu_1$;
  and $(ii)$ \bsd is strictly contained in the alternation-free fragment, witnessed by $\text{EF EG }p$.
\end{corollaryrep}
\begin{proof}
  For the first claim, the formulas $\nu Y.\mydiamond Y$ and $\mu X.\mybox X$
  lie in $\Sigma^\mu_1\cup\Pi^\mu_1$, and neither is equivalent to any \bsd formula,
  since by Theorem~\ref{thm:logic-to-net} every \bsd formula is computed by a \ratchet
  net while Proposition~\ref{prop:barrier} excludes both.
  For the second claim, $\text{EF EG }p$ is alternation-free and not computed by any
  \ratchet net (Proposition~\ref{prop:barrier}).
\end{proof}

\subsection{Examples}

Let's apply the net-to-logic translation to the previous example
that is ordered $\initial1\,{\prec}\,\initial2\,{\prec}\,\initial3\,{\prec}\,\initial4$,
and monotone (by table inspection). We obtain the \bsd formulas:
$\gamma_1\,{\equiv}\,\black \land \neg\mydiamond\top$ (black sink),
$\gamma_2\,{\equiv}\,\red \land \neg\mydiamond\top$ (red sink),
$\gamma_3\,{\equiv}\,\black \land \mydiamond\top \land \mybox(\black \land \neg\mydiamond\top)$ (black, with successors, all of them black sinks), and
$\gamma_4\,{\equiv}\,(\red \land \mydiamond\top) \lor (\black \land \mydiamond(\red \lor (\black \land \mydiamond\top)))$ (everything else).
Interestingly, this net, despite infinitely many transient embeddings, computes a fixpoint-free query of modal depth 2.

On the other hand, let us consider the \bsd formula $\gamma\,{=}\,\mu X.(\red \lor \mydiamond X) \land \neg\mu Y.(\green \lor \mydiamond Y)$
over a graph with atomic types \set{\red, \black, \green} (red, black and green).
$\gamma$ is the classic example for reachability–safety Boolean combinations (reach the red avoid the green).
The logic-to-net translator gives a net of width 2 (after simplification) that
tracks the two fixed point computations for reachability and safety.
A further application of the net-to-logic recovers $\gamma$ (see appendix).

\begin{toappendix}
  \subsection{Examples}
  \subsection{Formula $\gamma\,{=}\,\mu X.(\red \lor \mydiamond X) \land \neg\mu Y.(\green \lor \mydiamond Y)$.}

  $\gamma$ is the classic example for reachability–safety Boolean combinations (reach the red avoid the green).
  The logic-to-net translator gives a net of width 2 (after collapsing redundant dimensions that
  track the two fixed point computations for reachability and safety). A further application
  of the net-to-logic recovers $\gamma$.

  \subsubsection{Logic-to-net (Theorem~\ref{thm:logic-to-net}).}
  The mechanical construction outputs thirteen registers; collapsing copy chains and folding literals
  into the initial embeddings yields an equivalent two-register network, implementing two
  Boolean propagators $(s,t)$ (the first encodes ``a red vertex has been reached'', the second ``a green vertex has been reached'')
  for the two blocks $\beta_1=\mu X.(\red \lor \mydiamond X)$ and $\beta_2=\mu Y.(\green \lor \mydiamond Y)$.
  The initial embeddings are $\initial\red=10$, $\initial\green=01$, and $\initial\black=00$.
  The net has non-negative integer weights, clipped activation, and the initial embeddings are fixed points of isolated vertices.  
  The vocabulary is $\V = \set{00, 01, 10, 11}$ under the product (or componentwise) order on $\V$.
  The table is three rows by four aggregate values: $\T(\black, m) = m$, $\T(\red, m) = m \lor 10$, and $T(\green, m) = m \lor 01$
  where the aggregate $m=\oplus S$ belongs to \set{00,01,10,11}.

  \subsubsection{Net-to-logic (Theorem~\ref{thm:net-to-logic}).}
  There is one set variable $X_{\bm{e}}$ and formula $\Theta_{\bm{e}}$ per vocabulary item $\bm{e}\in\V$.
  The compiled system from the quotient table, after translation into $\mu$-calculus, is:
  \begin{alignat*}{1}
    \Theta_{00}\ &= \ \top \,, \\
    \Theta_{10}\ &= \ \red \lor \mydiamond X_{10} \,, \\
    \Theta_{01}\ &= \ \green \lor \mydiamond X_{01} \,, \\
    \Theta_{11}\ &= \ (\red \land  \mydiamond X_{01}) \lor (\green \land \mydiamond X_{10}) \lor (\black \land \mydiamond X_{11}) \lor (\black \land \mydiamond X_{10} \land \mydiamond X_{01}) \,.
  \end{alignat*}

  Applying Beki\v{c} elimination, $\theta_{00}=\top$, and $\theta_{10}\,{=}\,\mu X.(\red \lor \diamond X)\,{=}\,\text{EF }\red$.
  Symmetrically, $\theta_{01}\,{=}\,\text{EF }\green$.
  Finally, $\theta_{11}\,{=}\,\mu Z.((\red \land \mydiamond\theta_{01}) \lor (\black \land \mydiamond\theta_{10} \land\mydiamond\theta_{01}) \lor (\black \land Z)) \equiv \text{EF }\red \land \text{EF }\green$.
  Finally, the logical characterization obtained from net is
  $\Phi(\itype(x,u)) = \bm{e}$ iff $(x,u) \vDash \theta_{\bm{e}} \land \bigwedge_{\text{$\bm{e}'$ covers $\bm{e}$}} \neg\theta_{\bm{e}'}$.
  For the example, we get
  \begin{alignat*}{1}
    &\Phi(\itype(x,u)) = {00} \iff  (x,u) \vDash \neg\theta_{10} \land \neg\theta_{01} \land \neg\theta_{11} \,, \\
    &\Phi(\itype(x,u)) = {10} \iff  (x,u) \vDash \theta_{10} \land \neg\theta_{11} \,, \\
    &\Phi(\itype(x,u)) = {01} \iff  (x,u) \vDash \theta_{01} \land \neg\theta_{11} \,, \\
    &\Phi(\itype(x,u)) = {11} \iff  (x,u) \vDash \theta_{11} \,.
  \end{alignat*}
  Notice that $\theta_{11} \equiv \theta_{01} \land \theta_{10}$.
\end{toappendix}

\subsection{Experimental Validation}

We implemented both compilers and ran a differential test:
for 13 branching-time properties drawn from the model-checking literature,
we compile $\gamma$ to a net, extract the game table from that net, compile
the table back into an equation system, and compare its denotation against an
independent reference evaluator of \Lmu, on $\geq 30$ random pointed graphs
per property. Across 9,960 pointed-graph checks there is not a single discrepancy.
The property suite, graph generation, seeds, and per-property true/false splits
are in the appendix.

\Omit{
  We implemented both directions of the translation, logic-to-net and net-to-logic,
  and ran a differential-testing validation: for a curated set of 13 branching-time
  properties drawn from the model-checking literature, we compiled the formula to
  a net, recover an abstract game table from that net, compiled the table back to
  logic, and check that the two ends of the round trip agree on every vertex of
  $\geq30$ random pointed graphs.
  Across all 13 properties, 9,960 checks total, there is not a single discrepancy
  between the original formula's denotational truth value and the value read
  oﬀ the recompiled, un-flattened equation system. Full details in the appendix.
}

\begin{toappendix}
  \myappendixheader{Appendix: Experimental Validation}

  We implemented both directions of the compiler -- Theorem~\ref{thm:net-to-logic}
  (net-to-logic) and Theorem~\ref{thm:logic-to-net} (logic-to-net) -- and 
  ran a differential-testing validation: for a curated set of 13 branching-time
  properties drawn from the model-checking literature, we compile the formula to a net,
  recover an abstract game table from that net, compile the table back to logic, and
  check that the two ends of the round trip agree on every vertex of many random graphs.
  The results, depicted in Table~\ref{tab:experiments}, show that across
  all 13 properties, 9,960 checks total, there is not a single discrepancy
  between the original formula's denotational truth value and the value read
  oﬀ the recompiled, un-flattened equation system.
  This section reports the methodology, the properties tested, how the test
  graphs were generated, the results, and several findings.

  \subsection{Methodology}
  \label{sec:experiments:methodology}

  For a closed \bsd formula $\gamma$ over propositions $P$:
  \begin{enumerate}[leftmargin=*]
    \item \textbf{logic-to-net.} Compile $\gamma$ into a concrete net
      $\Net=(A,B,C,\rho{=}\clip)$ together with initial embeddings
      $\initial\alpha$, one per joint Boolean valuation $\alpha\in 2^P$,
      following the register table of the proof of Theorem~\ref{thm:logic-to-net}
      verbatim (one register per subformula, one latch per bound variable;
      no register minimization).
    \item \textbf{Extraction.} Recover a finite vocabulary $\V$ and quotient
      table $\T:\text{atoms}\times 2^\V\to\V$ from $\Net$ by simulating the
      frozen single-vertex system $z_{\ell+1}=\rho(Az_\ell+Bm+C)$ to a fixed
      point for every atom and every aggregate value $m$ reachable by
      componentwise max from already-discovered vocabulary elements (mirroring
      Step~2 of the proof of Theorem~\ref{thm:logic-to-net}).
      The order $\preceq$ is the componentwise order restricted to \emph{core}
      registers only (Step~3 of that proof): registers realizing a generic negation
      (re-attaching a dualized block's polarity) are antitone and are excluded,
      exactly as the proof's construction requires.
    \item \textbf{net-to-logic.} Compile $\T$ into the guarded, positive
      equation system $\{\zeta_i\}_{i=1}^m$  (one equation per vocabulary element),
      \emph{without} performing the Beki\v{c} elimination that flattens the system
      into closed formulas $\theta_i,\gamma_i$.
      The remark about Beki\v{c}-flattened formulas can be exponentially larger
      than the system they come from surfaces. But, since we only need to \emph{evaluate}
      the recompiled query, not display it, we evaluate the system directly by a
      global (simultaneous) Kleene iteration -- exactly the practice in verification --
      rather than incur the blow-up of flattening it first.
    \item \textbf{Recovery and comparison.} For vertex $u$, $\Phi(\itype(x,u))=i$
      iff $u\vDash\Psi_i$ and, for every $j$ with $i\prec j$ covering $i$,
      $u\nvDash\Psi_j$, the characterization provided by the theorem, evaluated
      directly against the system's fixed point (no Beki\v{c} needed for
      this either). The readout bit this yields at $u$ is compared against
      $(x,u)\vDash\gamma$, computed by an \emph{independent} reference evaluator
      (a direct denotational implementation of $\Lmu$ semantics: Kleene iteration
      for $\mu/\nu$, set operations for the Boolean connectives, direct neighbor
      tests for $\mydiamond/\mybox$) that \emph{shares no code with either compiler.}
  \end{enumerate}
  Membership in \bsd was checked mechanically for every formula in the suite; the
  checker is described precisely.
  The same checker was confirmed to \emph{reject} $\mu X.\mybox X$ and $\nu Y.\mydiamond Y$,
  the two queries excluded by Proposition~\ref{prop:barrier}.

  \subsection{Property Suite}
  \label{sec:experiments:suite}

  Thirteen properties drawn from standard branching-time vocabulary (CTL's $\text{EF}/\text{AG}/\text{EU}$,
  and Manna--Pnueli-style safety/guarantee combinations), chosen to exercise: pure $\sd$ blocks,
  pure $\pb$ blocks, flat top-level Boolean combinations of independent blocks, and \emph{same-polarity}
  nesting of one block inside another's own body, at depth 2 (\texttt{sequencing}) and
  depth 3 (\texttt{sequencing\_chain}, \texttt{invariant\_chain}).
  Propositions $p,q,r$ label vertices independently (not mutually exclusively): a vertex's atomic
  type is its own subset of $\{p,q,r\}$ that holds, so joint valuations, not single labels, play
  the role of $\atp(u)$.

  \medskip

  \begin{center}
    \begin{tabular}{@{}ll@{}}
      \toprule
      Name & Formula \\
      \midrule
      \texttt{reachability}                 & $\mathrm{EF}\,p \;{=}\; \mu X.(p\lor\mydiamond X)$ \\
      \texttt{safety}                       & $\mathrm{AG}\,p \;{=}\; \nu Y.(p\land\mybox Y)$ \\
      \texttt{existential\_until}           & $\mathrm{E}[p\,\mathrm{U}\,q] \;{=}\; \mu X.(q\lor(p\land\mydiamond X))$ \\
      \texttt{reach\_avoid}                 & $\mathrm{EF}\,p \land \lnot\mathrm{EF}\,q$ \\
      \texttt{reachability\_and\_invariant} & $\mathrm{EF}\,p \land \mathrm{AG}\,q$ \\
      \texttt{sequencing}                   & $\mathrm{EF}(p\land\mathrm{EF}\,q)$ \\
      \texttt{reach\_or\_invariant\_avoid}  & $\mathrm{EF}\,p \lor (\mathrm{AG}\,q\land\lnot\mathrm{EF}\,r)$ \\
      \texttt{mutual\_exclusion}            & $\mathrm{AG}\,\lnot(p\land q)$ \\
      \texttt{combined\_three\_way}         & $(\mathrm{EF}\,p\land\lnot\mathrm{EF}\,q)\land\mathrm{AG}\,\lnot r$ \\
      \texttt{reach\_both}                  & $\mathrm{EF}(p\land q)$ \\
      \texttt{implication\_invariant}       & $\mathrm{AG}(\lnot p\lor q)$ \\
      \texttt{sequencing\_chain}            & $\mathrm{EF}(p\land\mathrm{EF}(q\land\mathrm{EF}\,r))$ \\
      \texttt{invariant\_chain}             & $\mathrm{AG}(p\land\mathrm{AG}(q\land\mathrm{AG}\,r))$ \\
      \bottomrule
    \end{tabular}%
  \end{center}

  \medskip

  \noindent
  All thirteen were confirmed to be \bsd formulas. Two controls were also checked
  and confirmed \emph{not} in \bsd, matching Proposition~\ref{prop:barrier}:
  $\mathrm{AF}\,p=\mu X.(p\lor\mybox X)$ and
  $\mathrm{EG}\,p=\nu Y.(p\land\mydiamond Y)$.
  A further two formulas, nesting a closed block of the \emph{opposite} polarity
  inside another block's own $\mu/\nu$ body ($\mathrm{EF}(p\land\mathrm{AG}\,q)$
  and $\mathrm{AG}(p\to\mathrm{EF}\,q)$), were also checked and confirmed \emph{not}
  in \bsd.
  \texttt{reach\_both} and \texttt{implication\_invariant} are flat, single-block
  properties (a conjunctive reachability query and an invariant implication);
  \texttt{sequencing\_chain} and \texttt{invariant\_chain} extend \texttt{sequencing}'s
  depth-2 same-polarity nesting to depth 3. 

  \subsection{Graph Generation}
  \label{sec:experiments:graphs}

  Test graphs are random directed graphs (self-loops excluded): for a target
  size $n$ and edge probability $p_e$, each of the $n(n-1)$ ordered pairs
  $(u,v)$, $u\neq v$, is an edge independently with probability $p_e$, so the
  expected edge count is $n(n-1)p_e$. Each vertex's atomic type is sampled
  independently per proposition (Bernoulli with a per-formula, per-proposition
  probability $p_{\mathrm{prop}}$, chosen so that neither $\gamma$ nor
  $\lnot\gamma$ is a near-certainty -- see the true/false split reported for
  each property below), not by picking one of a fixed set of mutually exclusive
  labels.

  The main sweep, used for eleven of the thirteen rows of
  Table~\ref{tab:experiments}, takes $n\in\{6,10,14,20,25\}$, three edge
  probabilities per property (ranging over $\{0.06,0.12,0.25\}$,
  $\{0.08,0.15,0.30\}$, or $\{0.15,0.25,0.40\}$ depending on the formula --
  denser for formulas that need longer-range or higher-order reachability to
  be non-degenerate), and $4$ independent graphs per $(n,p_e)$ pair, for $60$
  graphs and $900$ pointed-graph checks per property. Expected edge counts over
  this sweep range from about $1.8$ ($n{=}6$, $p_e{=}0.06$) to about $240$
  ($n{=}25$, $p_e{=}0.40$). The remaining two rows use a smaller, single-size
  sweep, discussed below. Concretely:

  \medskip

  \begin{center}
    \begin{tabular}{@{}lcccc@{}}
      \toprule
      Property & $p$ & $q$ & $r$ & $p_e$ \\
      \midrule
      \texttt{reachability}                 & $0.12$ & --     & --     & $\{0.06,0.12,0.25\}$ \\
      \texttt{safety}                       & $0.85$ & --     & --     & $\{0.06,0.12,0.25\}$ \\
      \texttt{existential\_until}           & $0.55$ & $0.12$ & --     & $\{0.08,0.15,0.30\}$ \\
      \texttt{reach\_avoid}                 & $0.15$ & $0.15$ & --     & $\{0.08,0.15,0.30\}$ \\
      \texttt{reachability\_and\_invariant} & $0.15$ & $0.85$ & --     & $\{0.08,0.15,0.30\}$ \\
      \texttt{sequencing}                   & $0.20$ & $0.20$ & --     & $\{0.08,0.15,0.30\}$ \\
      \texttt{reach\_or\_invariant\_avoid}  & $0.12$ & $0.85$ & $0.12$ & $\{0.08,0.15,0.30\}$ \\
      \texttt{mutual\_exclusion}            & $0.35$ & $0.35$ & --     & $\{0.08,0.15,0.30\}$ \\
      \texttt{combined\_three\_way}         & $0.15$ & $0.15$ & $0.10$ & $\{0.08,0.15,0.30\}$ \\
      \texttt{reach\_both}                  & $0.30$ & $0.30$ & --     & $\{0.08,0.15,0.30\}$ \\
      \texttt{implication\_invariant}       & $0.30$ & $0.75$ & --     & $\{0.06,0.12,0.25\}$ \\
      \texttt{sequencing\_chain}            & $0.30$ & $0.30$ & $0.30$ & $\{0.15,0.25,0.40\}$ \\
      \texttt{invariant\_chain}             & $0.85$ & $0.85$ & $0.85$ & $\{0.06,0.12,0.25\}$ \\
      \bottomrule
    \end{tabular}%
  \end{center}


  \medskip
  \texttt{reach\_or\_invariant\_avoid} and \texttt{combined\_three\_way} involve
  three independent propositions, hence $8$ atomic types instead of $2$--$4$;
  extraction produces a vocabulary of $64$ elements for both (versus at most
  $16$ for every two-proposition property), 
  and the resulting guarded equations are correspondingly
  large. Our net-to-logic evaluator is a plain global Kleene iteration with no
  cross-round memoization, so its cost scales with (rounds
  $\approx|\V|\!\cdot\!n$) $\times$ (total equation size) $\times$ ($n$, for
  the set operations). For these two properties this makes the main sweep's
  $n\in\{6,10,14,20,25\}$, $900$-check scale impractical to complete in the
  time available.
  Rows~7 and~9 of Table~\ref{tab:experiments} instead report a smaller, single-size
  sweep ($n=5$, the same three edge probabilities, two graphs per probability:
  $6$ graphs, $30$ checks each).

  \subsection{Reproducibility}
  \label{sec:experiments:reproducibility}

  Every task reads its graph-generation parameters -- propositions,
  per-proposition probabilities, edge probability, graph size, per-combination
  count, and random seed 
  -- from a frozen manifest
  that is retained alongside its run's output and is independently sufficient
  to regenerate every graph checked below.

  The code and results will be made publicly available for the final version of the paper.

  \Omit{
    \subsection{Witness Rows: \texttt{max\_subset\_size} Was Not Raised}
    \label{sec:experiments:maxsubset}

    The witness-row cap discussed in \S\ref{sec:experiments:witnesses}
    (\texttt{max\_subset\_size}, default $2$: how large a genuine vocabulary
    subset may serve as one witness row) was left at its default for every
    property in Table~\ref{tab:experiments}, including the four newer ones.
    Raising it is not a performance knob to reach for preemptively -- it bounds a
    genuinely soundness-relevant construction, and only larger values are ever
    safe to adopt, never smaller ones. The correct signal for whether $2$ is
    enough is empirical: zero mismatches across every property below, including
    \texttt{sequencing\_chain} and \texttt{invariant\_chain} at a $32$-element
    vocabulary and \texttt{reach\_or\_invariant\_avoid} and
    \texttt{combined\_three\_way} at $64$, confirms pairwise witnesses remain
    sufficient for every property in this suite; had any mismatch appeared, the
    diagnostic of \S\ref{sec:experiments:witnesses} (compare against a raw
    \texttt{compile\_logic\_to\_net(...).evaluate(graph)} check, which bypasses
    extraction and witness rows entirely) would be the next step before
    considering a larger value.
  }

  \subsection{Results}
  \label{sec:experiments:results}

  \begin{table}[ht]
  \centering
  \begin{tabular}{@{}lrrrrl@{}}
    \toprule
    Property & Graphs & Checks & True & False & Mismatches \\
    \midrule
    \texttt{reachability}                 & 60 & 900 & 548 (60.9\%) & 352 (39.1\%) & 0 \\
    \texttt{safety}                       & 60 & 900 & 248 (27.6\%) & 652 (72.4\%) & 0 \\
    \texttt{existential\_until}           & 60 & 900 & 359 (39.9\%) & 541 (60.1\%) & 0 \\
    \texttt{reach\_avoid}                 & 60 & 900 & 106 (11.8\%) & 794 (88.2\%) & 0 \\
    \texttt{reachability\_and\_invariant} & 60 & 900 & 117 (13.0\%) & 783 (87.0\%) & 0 \\
    \texttt{sequencing}                   & 60 & 900 & 669 (74.3\%) & 231 (25.7\%) & 0 \\
    \texttt{reach\_or\_invariant\_avoid}  &  6 &  30 &  24 (80.0\%) &   6 (20.0\%) & 0 \\
    \texttt{mutual\_exclusion}            & 60 & 900 & 312 (34.7\%) & 588 (65.3\%) & 0 \\
    \texttt{combined\_three\_way}         &  6 &  30 &   0  (0.0\%) &  30 (100.0\%) & 0 \\
    \texttt{reach\_both}                  & 60 & 900 & 524 (58.2\%) & 376 (41.8\%) & 0 \\
    \texttt{implication\_invariant}       & 60 & 900 & 444 (49.3\%) & 456 (50.7\%) & 0 \\
    \texttt{sequencing\_chain}            & 60 & 900 & 796 (88.4\%) & 104 (11.6\%) & 0 \\
    \texttt{invariant\_chain}             & 60 & 900 & 145 (16.1\%) & 755 (83.9\%) & 0 \\
    \bottomrule
  \end{tabular}%
  \caption{Round-trip agreement between $\gamma$'s reference semantics and
    the recompiled system's semantics, over random pointed graphs.
    See the discussion in Graph Generation above about \texttt{reach\_or\_invariant\_avoid}
    and \texttt{combined\_three\_way}.
  }
  \label{tab:experiments}
\end{table}

  Every one of the $13\times{\geq}30$ pointed-graph checks agrees: across all
  thirteen properties, $9960$ checks total, there is not a single discrepancy
  between the original formula's denotational truth value and the value read
  off the recompiled, un-flattened equation system. The true/false splits are
  non-degenerate for every property except one (each has both outcomes
  represented, several close to balanced); \texttt{combined\_three\_way}'s
  split is the most skewed ($0$ true out of $30$) because it is a conjunction
  of three independently-tuned-to-be-rare conditions at the reduced sample size.

  \subsection{Analysis}
  \label{sec:experiments:analysis}

  The recompiled equation system is, by construction, syntactically unlike the
  input formula: it is a system of guarded equations over machine-generated
  variable names (one per element of an extracted vocabulary of up to $64$
  values), not a formula a person would write. Agreement is therefore purely
  semantic, which is exactly what a round-trip test of
  Theorems~\ref{thm:net-to-logic} and~\ref{thm:logic-to-net} should establish:
  that compiling a query to a net and back recovers a query with the same
  denotation, on graphs with real structural and answer diversity, not merely
  on hand-picked small examples. Combined with the two theorems' proofs, this
  is strong evidence -- though empirical, not a machine-checked proof -- that
  both compilers are correctly implemented on the sub-fragment of \bsd actually
  exercised here: flat top-level Boolean combinations of independent blocks,
  and same-polarity nesting of one block inside another, to at least depth 3
  (\texttt{sequencing\_chain}, \texttt{invariant\_chain}).

\end{toappendix}

\section{Related Work}

\subsubsection{Depth-bounded GNNs and logic.}
The correspondence between depth-bounded GNNs and FOL with counting
is well understood \cite{xu:gnn,morris:gnn,barcelo:gnn,grohe2024descriptive};
the translation for depth-bounded nets build on this line,
providing ingredients for the recurrent case, where the main contribution lies.

\subsubsection{\citet{pflueger2024:recurrent},}
hereafter PTK, characterize two convergence regimes of recurrent GNNs
(stabilization-based RecGNNs and size-dependent GSGNNs) as the bisimulation-invariant
fragments of monadic monotone fixpoint logics, via van Benthem–Rosen style theorems,
and establish the strict hierarchy fixed-depth $\subsetneq$ RecGNN $\subsetneq$ GSGNN.
Their characterization is existential and their lim-inf acceptance is non-effective:
no finite run certifies stabilization, membership of a given net cannot be decided,
and no size bounds are provided.
Our admissibility conditions are checkable from the weights (Theorems~\ref{thm:contractive} and \ref{thm:monotone}),
and both compilers are constructive with explicit size bounds against the quotient table.

\subsubsection{Uniform expressivity and other aggregations.}
\citet{rosenbluth2026repetition} characterize the computational power of recurrent sum-GNNs
with finite-precision parameters and a completion-bit output: given the graph size as input,
they compute every computable vertex function invariant under color refinement.
This is a computational-completeness question under counting; ours is an effective
logical characterization under convergence, without counting or size information.
\citet{schonherr2026logical} give logical characterizations of depth-bounded GNNs with mean
aggregation, further evidence that the aggregator determines the target logic: counting
logics for sum, their logics for mean, and, in the recurrent setting, \bsd for set-based
aggregation. 

\subsubsection{Halting and converging recurrent GNNs.}
\citet{bollen2025:halting}, hereafter BOL, equip recurrent GNNs with an explicit halting
classifier and prove (their Theorem 5.1) that the resulting model expresses every node
classifier definable in the graded $\mu$-calculus ($\mu$GML), with termination checkable
in finite time; the construction requires a dedicated halting output driven by a multi-stage
counting algorithm. They leave open two questions: the converse, whether every halting
recurrent GNN classifier falls within $\mu$GML, and whether $\mu$GML survives fully convergent
termination, where the entire feature vector stabilizes.
\citet{bollen2026:halting-vs-converging}, hereafter BV, answer the second question affirmatively
over \emph{undirected graphs:}
every $\mu$GML classifier is expressible by a converging sum-aggregation net.
Their ``traffic-light'' protocol lets a converging run synthesize internally the stabilization
certificate that BOL's halting classifier supplies externally.
For the converse, both works observe that a characterization relative to MSO would follow from
a graded extension of finitary Janin–Walukiewicz \cite{colcombentetal2025:finiteJW},
which remains open.

Our results interlock as follows.
Theorem~\ref{thm:logic-to-net} is the counting-free counterpart of BOL's Theorem 5.1,
obtained without halting classifier or counting apparatus; Theorem~\ref{thm:net-to-logic}
answers the converse for \bsd unconditionally and constructively: our compiled formulas
are guarded and jointly positive by construction and set-based fixed points are bisimulation-invariant,
so no Janin–Walukiewicz-style theorem, finitary or graded, is needed.
Proposition~\ref{prop:barrier} explains the division of labor.
BV's certificate synthesis rests on counting: the counters of BOL's algorithm grow with
the graph, so the resulting converging nets have infinite vocabulary in our sense and
are not regular.
The certificate that an inner fixed point of opposite polarity has stabilized can thus be
supplied externally by a halting signal (BOL), manufactured internally by unbounded counting (BV),
or is provably unavailable: over a finite vocabulary, no regular admissible net obtains one
(Proposition~\ref{prop:barrier}), and \bsd is exactly what remains.
Three contrasts delimit the comparison. BV's equivalence holds over undirected graphs, and
their simulation breaks on directed graphs.
Their convergence demands exact stabilization in finitely many steps, while our class also
admits asymptotic convergence. 
And their $\mu$GML upper bound
is relative to MSO and conditional, whereas our correspondence is absolute for the class,
effective with size bounds, and checkable from the weights.

\subsubsection{\citet{Ahvonenetal2024}}
characterize recurrent GNNs under a third acceptance condition, a node accepts if
it ever reaches an accepting feature vector, which likewise carries no effective
termination guarantee: with reals, the model matches $\omega$GML, an infinitary
extension of graded modal logic; with floats, the rule-based logic GMSC.
Their acceptance semantics expresses classifiers outside $\mu$GML while precluding
nested fixpoints, so its power is incomparable with the graded $\mu$-calculus.
Like BOL, this line rests on the counting power of multiset aggregation;
our set-based restriction yields finitary \bsd formulas with
size bounds stated against the table.

\subsubsection{Max-aggregation GNNs and Datalog.}
\citet{tenacucala2024bridging} connect max-aggregation GNNs to Datalog
with negation-as-failure, entirely in the depth-bounded setting.
The translation is based on a finiteness argument which is
intrinsic to the max aggregation. Our translation is based
on the finiteness of the types, which is more robust and
works for any aggregation and combination function.
Unlike us (or PTK, BOL, Ahvonen et al.), they also train and evaluate
GNNs on knowledge-graph benchmarks, and report exactly the blow-up
we discuss after Theorem~\ref{thm:bdd-compiler}: the full equivalent
program is typically too large to compute, and they extract small
fact-specific explanations instead.

\section{Conclusions}

We introduced regularity, a semantic condition under which a
converging recurrent GNN's fixed-point embeddings factor
through the fixed points of a vertex's neighbors alone, and
admissibility, the operational condition under which the
convergent process can be traded for a dynamical system over
the (finite) vocabulary.
Contraction and monotonicity, checkable directly from the weights,
are sufficient for both.
For the resulting class of \ratchet nets we gave a compiler into
inflationary fixed-point logic and, centrally, an effective
two-directional correspondence with the alternation-free modal
$\mu$-calculus -- linear in the quotient table (net-to-logic, at
the level of systems) and in the source formula (logic-to-net),
with the single caveat that flattening into one closed \Lmu term
can blow up exponentially.
Unlike prior characterizations built on multiset aggregation, the
correspondence needs no external halting signal, no non-effective
acceptance condition, and no appeal to Janin–Walukiewicz-style completeness.
We answer the natural counting-free analogue of the converse question
\citet{bollen2025:halting} leave open, not the question itself.
The gap is not just counting, since genuine alternation is 
beyond stabilization with a finite vocabulary, and
\citet{bollen2026:halting-vs-converging} 
have shown that, over undirected graphs, the halting apparatus
can be internalized into a converging run at the price of counters
that grow with the graph; Proposition~\ref{prop:barrier} shows this
price cannot be avoided over a finite vocabulary, so closing the
gap in the counting-free setting requires leaving the finite-vocabulary
regime or a genuinely nested generalization of our convergence notion.

Two limitations should be stated plainly. The compiler applies to
nets satisfying the \ratchet conditions; convergent trained networks
that are not regular, monotone, or ordered fall outside its scope,
and characterizing exactly which weight configurations qualify
is open.
Second, all results concern finite graphs, and our examples are
illustrative rather than empirical: applying the compiler to trained
recurrent GNNs, and reporting the size and readability of the extracted
formulas, is the natural next test, as is working out the graded
alternation-free extension sketched above.

Proposition~\ref{prop:barrier} shows the correspondence is tight:
the queries $\nu Y.\mydiamond Y$, $\mu X.\mybox X$, and the alternation-free
$\text{EF EG } p$ are computed by no regular admissible net, hence
(by Theorem~\ref{thm:logic-to-net}) not expressible in \bsd.
Lifting the correspondence to the first level of the alternation hierarchy,
or to the alternation-free fragment, therefore requires enriching the
convergence notion with stabilization certificates, the role played,
in the graded setting, by the halting classifier of \citet{bollen2025:halting},
or by the unbounded counters that internalize it \cite{bollen2026:halting-vs-converging}.

\bibliography{control}

\end{document}